\documentclass{article}
\usepackage[accepted]{icml2021}
\usepackage{times}

\usepackage{hyperref}
\usepackage{url}
\usepackage{booktabs}       
\usepackage{amsfonts}       
\usepackage{nicefrac}       
\usepackage{microtype}      
\usepackage{helvet}
\usepackage{courier}
\usepackage{url}
\usepackage{graphicx}
\usepackage{blkarray}
\usepackage{amsmath,amssymb}
\usepackage{amsthm}
\usepackage[flushleft]{threeparttable}
\usepackage{array,booktabs,makecell}
\usepackage[font=small]{caption}
\usepackage{natbib}
\usepackage[capitalize]{cleveref}
\usepackage{autonum}
\usepackage{mathtools}
\usepackage[colorinlistoftodos]{todonotes}
\usepackage{enumitem}
\usepackage[english]{babel}
\usepackage[parfill]{parskip}
\usepackage{bm}
\usepackage{mathrsfs}
\usepackage[utf8]{inputenc} 
\usepackage[T1]{fontenc}    
\usepackage{hyperref}       
\usepackage{url}            
\usepackage{booktabs}       
\usepackage{nicefrac}       
\usepackage{mathdots}

\usepackage{tikz}
\usetikzlibrary{arrows,shapes,snakes,automata,backgrounds,petri}

\usepackage{makecell}
\usepackage{array}
\usepackage{subfigure} 
\usepackage{dblfloatfix}
\usepackage{tablefootnote}
\usepackage[flushleft]{threeparttable}
\usepackage[font=small]{caption}
\usepackage{multirow}
\usepackage{wrapfig}
\usepackage{multicol}
\usepackage{caption}

\usepackage[scr=esstix]{mathalfa}
\usepackage{mathtools}

\usepackage{float}
\graphicspath{{figures/}}

\usepackage{algorithmic}
\usepackage{algorithm}
\usepackage{changepage}
\newtheorem{theorem}{Theorem}
\newtheorem{corollary}{Corollary}[theorem]
\newtheorem{lemma}{Lemma}

\newtheorem{assumption}{Assumption}
\newtheorem{proposition}{Proposition}

\DeclareMathOperator*{\argmin}{arg\,min}

\usepackage{microtype}
\usepackage{graphicx}
\usepackage{subfigure}
\usepackage{booktabs} 

\usepackage{hyperref}

\icmltitlerunning{Tesseract: Tensorised Actors for Multi-Agent Reinforcement Learning }

\author{}

\begin{document}	

\setlength{\abovedisplayskip}{4pt}
\setlength{\belowdisplayskip}{4pt}

\twocolumn[
\icmltitle{Tesseract: Tensorised Actors for Multi-Agent Reinforcement Learning}









\begin{icmlauthorlist}
\icmlauthor{Anuj Mahajan}{ox}
\icmlauthor{Mikayel Samvelyan}{uc}
\icmlauthor{Lei Mao}{nv}
\icmlauthor{Viktor Makoviychuk}{nv}
\icmlauthor{Animesh Garg}{nv}
\icmlauthor{Jean Kossaifi}{nv}
\icmlauthor{Shimon Whiteson}{ox}
\icmlauthor{Yuke Zhu}{nv}
\icmlauthor{Animashree Anandkumar}{nv}
\end{icmlauthorlist}

\icmlaffiliation{ox}{University of Oxford}
\icmlaffiliation{nv}{NVIDIA}
\icmlaffiliation{uc}{University College London}

\icmlcorrespondingauthor{Anuj Mahajan}{anuj.mahajan@cs.ox.ac.uk}

\icmlkeywords{Machine Learning, ICML}

\vskip 0.3in
]



\printAffiliationsAndNotice{}  

\begin{abstract}
Reinforcement Learning in large action spaces is a challenging problem. Cooperative multi-agent reinforcement learning (MARL) exacerbates matters by imposing various constraints on communication and observability.
In this work, we consider the fundamental hurdle affecting both value-based and policy-gradient approaches: an exponential blowup of the action space with the number of agents. For value-based methods, it poses challenges in accurately representing the optimal value function. For policy gradient methods, it makes training the critic difficult and exacerbates the problem of the \emph{lagging} critic. We show that from a learning theory perspective, both problems can be addressed by accurately representing the associated action-value function with a low-complexity  hypothesis class. This requires accurately modelling the agent interactions in a sample efficient way. To this end, we propose a novel tensorised formulation of the Bellman equation. This gives rise to our method \textsc{Tesseract}, which views the $Q$-function as a tensor whose modes correspond to the action spaces of different agents. Algorithms derived from \textsc{Tesseract} decompose the $Q$-tensor across agents and utilise low-rank tensor approximations to model  agent interactions relevant to the task.  We provide PAC analysis for \textsc{Tesseract}-based algorithms and highlight their relevance to the class of rich observation MDPs. Empirical results in different domains confirm \textsc{Tesseract}'s gains in sample efficiency predicted by the theory.
\end{abstract}
	
\section{Introduction}
\label{intro}
Many real-world problems, such as swarm robotics and autonomous vehicles, can be formulated as multi-agent reinforcement learning (MARL) \cite{bucsoniu2010multi} problems. MARL introduces several new challenges that do not arise in single-agent reinforcement learning (RL), including  exponential growth of the action space in the number of agents. This affects multiple aspects of learning, such as credit assignment~\cite{foerster2018counterfactual}, gradient variance~\cite{lowe2017multi} and exploration~\cite{mahajan2019maven}. In addition, while the agents can typically be trained in a centralised manner, practical constraints on observability and communication after deployment imply that decision making must be decentralised, yielding the extensively studied setting of centralised training with decentralised execution (CTDE).

Recent work in CTDE-MARL can be broadly classified into value-based methods and  actor-critic methods. Value-based methods \cite{sunehag_value-decomposition_2017,rashid2018qmix, son2019qtran, wang2020qplex, yao2019smix} typically enforce decentralisability by modelling the joint action $Q$-value such that the argmax over the joint action space can be tractably computed by local maximisation of per-agent utilities. However,  constraining the representation of the $Q$-function can interfere with exploration, yielding provably suboptimal solutions \cite{mahajan2019maven}. Actor-critic methods \cite{lowe2017multi, foerster2018counterfactual,wei2018multiagent} typically use a centralised critic to estimate the gradient for a set of decentralised policies. In principle, actor-critic methods can satisfy CTDE without incurring suboptimality, but in practice their performance is limited by the accuracy of the critic, which is hard to learn given exponentially growing action spaces. This can exacerbate the problem of the \textit{lagging} critic \cite{kondathesis}. Moreover, unlike the single-agent setting, this problem cannot be fixed by increasing the critic's learning rate and number of training iterations. Similar to these approaches, an exponential blowup in the action space also makes it difficult to choose the appropriate class of models which strike the correct balance between expressibility and learnability for the given task.

In this work, we present new theoretical results that show how the aforementioned approaches can be improved such that they accurately represent the joint action-value function whilst keeping the complexity of the underlying hypothesis class low. This translates to accurate, sample efficient modelling of long-term agent interactions.
 
In particular, we propose \textsc{Tesseract} (derived from "Tensorised Actors"), a new framework that leverages tensors for MARL. Tensors are high dimensional analogues of matrices that offer rich insights into representing and transforming data.  The main idea of \textsc{Tesseract} is to view the output of a joint $Q$-function as a tensor whose modes correspond to the actions of the different agents. We thus formulate the Tensorised Bellman equation, which offers a novel perspective on the underlying structure of a multi-agent problem. In addition, it enables the derivation of algorithms that decompose the $Q$-tensor across agents and utilise low rank approximations to model relevant agent interactions.
 
 Many real-world tasks (e.g., robot navigation) involve high dimensional observations but can be completely described by a low dimensional feature vector (e.g., a 2D map suffices for navigation). For value-based \textsc{Tesseract} methods, maintaining a tensor approximation with rank  matching the intrinsic task dimensionality\footnote{We define intrinsic task dimensionality (ITD) as the minimum number of dimensions required to describe an environment} helps learn a compact approximation of the true $Q$-function (alternatively MDP-dynamics for model based methods).
 In this way, we can avoid the suboptimality of the learnt policy while remaining sample efficient. Similarly, for actor-critic methods, \textsc{Tesseract} reduces the critic's learning complexity while retaining its accuracy, thereby mitigating the lagging critic problem. Thus, \textsc{Tesseract} offers a natural spectrum for trading off accuracy with computational/sample complexity. 

To gain insight into how tensor decomposition helps improve sample efficiency for MARL, we provide theoretical results for model-based \textsc{Tesseract} algorithms and show that the underlying joint transition and reward functions can be efficiently recovered under a PAC framework (in samples polynomial in accuracy and confidence parameters). 
We also introduce a tensor-based framework for CTDE-MARL that opens new possibilities for developing efficient classes of algorithms. Finally, we explore the relevance of our framework to rich observation MDPs.
Our main contributions are:
\begin{enumerate}
    \item A novel tensorised form of the Bellman equation;
    \item \textsc{Tesseract}, a method to factorise the action-value function based on tensor decomposition, which can be used for any factored action space;
    \item PAC analysis and error bounds for model based \textsc{Tesseract} that show an exponential gain in sample efficiency of $O(|U|^{n/2})$; and
    \item  Empirical results illustrating the advantage of \textsc{Tesseract} over other methods and detailed techniques for making tensor decomposition work for deep MARL.
\end{enumerate}
\section{Background}
\label{background}

\paragraph{Cooperative MARL settings}
In the most general setting, a fully cooperative multi-agent task can be modelled as a multi-agent partially observable MDP  (M-POMDP)~\cite{messias2011mpomdp}. An M-POMDP is formally defined as a tuple $G=\left\langle S,U,P,r,Z,O,n,\gamma\right\rangle$. 
$S$ is the state space of the environment. At each time step $t$, every agent $ i \in \mathcal{A} \equiv \{1,...,n\}$ chooses an action $u^i \in U$ which forms the joint action $\mathbf{u}\in\mathbf{U}\equiv U^n$.
$P(s'|s,\mathbf{u}):S\times\mathbf{U}\times S\rightarrow [0,1]$ is the state transition function. $r(s,\mathbf{u}):S \times \mathbf{U} \rightarrow [0,1]$ is the reward function shared by all agents and $\gamma \in [0,1)$ is the discount factor. 
\begin{wrapfigure}{r}{0.6\linewidth}
\centering
\includegraphics[width=\linewidth]{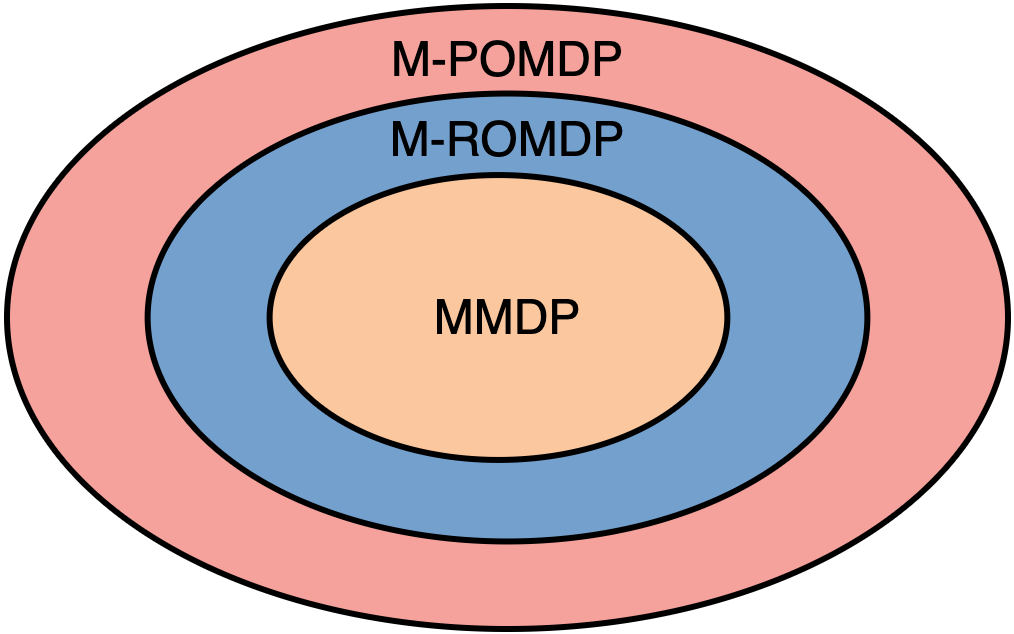}
\caption{Different settings in MARL \label{settings_marl}}
\end{wrapfigure}
An M-POMDP is \textit{partially observable}
\citep{kaelbling1998planning}: each agent does not have access to the full state and instead samples observations $z\in Z$ according to observation distribution $O(s):S\rightarrow \mathcal{P}(Z)$.
The action-observation history for an agent $i$ is $\tau^i\in T\equiv(Z\times U)^*$.
We use $u^{-i}$ to denote the action of all the agents other than $i$ and similarly for the policies $\pi^{-i}$. Settings where the agents cannot exchange their action-observation histories with others and must condition their policy solely on local trajectories, $\pi^i(u^i|\tau^i):T\times U\rightarrow [0,1]$, are referred to as a decentralised partially observable MDP (Dec-POMDP)~\citep{oliehoek_concise_2016}. 
When the observations have additional structure, namely the joint observation space is partitioned w.r.t.\ $S$, i.e., $\forall s_1,s_2 \in S \land {z} \in Z, P({z}|s_1)>0 \land s_1\neq s_2 \implies P({z}|s_2)=0$, we classify the problem as a multi-agent richly observed MDP (M-ROMDP)~\citep{azizzadenesheli2016reinforcement}. For both M-POMDP and M-ROMDP, we assume $|Z|>>|S|$, thus for this work, we assume a setting with no information loss due to observation but instead, redundancy across different observation dimensions. Such is the case for many real world tasks like 2D robot navigation using observation data from different sensors.
Finally, when the observation function is a bijective map $O: S\to Z$, we refer to the scenario as a multi-agent MDP (MMDP)~\cite{boutilier1996planning}, which can simply be denoted by the tuple : $\left\langle S,U,P,r,n,\gamma\right\rangle$. \cref{settings_marl} gives the relation between different scenarios for the cooperative setting. For ease of exposition, we present our theoretical results for the MMDP case, though they can easily be extended to other cases by incurring additional sample complexity. 

The joint \textit{action-value function} given a policy $\pi$ is defined as: $Q^\pi(s_t, \mathbf{u}_t)=\mathbb{E}_{s_{t+1:\infty},\mathbf{u}_{t+1:\infty}} \left[\sum^{\infty}_{k=0}\gamma^kr_{t+k}|s_t,\mathbf{u}_t\right]$.
The goal is to find the optimal policy $\pi^{*}$ corresponding to the optimal action value function $Q^*$. For the special learning scenario called Centralised Training with Decentralised Execution (CTDE), the learning algorithm has access to the action-observation histories of all agents and the full state during training phase. However, each agent can only condition on its own local action-observation history $\tau^i$  during the decentralised execution phase.

\paragraph{Reinforcement Learning Methods\label{subsec: rl}}
Both value-based and actor-critic methods for reinforcement learning (RL) rely on an estimator for the action-value function $Q^\pi$ given a target policy $\pi$. $Q^\pi$ satisfies the (scalar)-Bellman expectation equation:
$ Q^{\pi}(s, \mathbf{u}) = r(s, \mathbf{u})+ \gamma\mathbb{E}_{s',\mathbf{u}'}[Q^{\pi}(s', \mathbf{u}')],$ which can equivalently be written in vectorised form as:
\begin{align}
\label{eq:bell}
Q^{\pi} = R + \gamma P^{\pi}Q^{\pi}, 
\end{align}
where $R$ is the mean reward vector of size $S$, $P^{\pi}$ is the transition matrix. The operation on RHS 
$\mathcal{T}^{\pi}(\cdot) \triangleq R + \gamma P^{\pi}(\cdot)$ is the Bellman expectation operator for the policy $\pi$. In \cref{method} we generalise \cref{eq:bell} to a novel tensor form suitable for high-dimensional and multi-agent settings. For large state-action spaces function approximation is used to estimate $Q^\pi$. A parametrised approximation $Q^\phi$ is usually trained using the bootstrapped target objective derived using the samples from $\pi$ by minimising the mean squared temporal difference error:
$
    \mathbb{E}_\pi[(r(s, \mathbf{u})+ \gamma Q^{\phi}(s', \mathbf{u}')-Q^{\phi}(s, \mathbf{u}))^2].
$
Value based methods use the $Q^\pi$ estimate to derive a behaviour policy which is iteratively improved using the policy improvement theorem \cite{sutton2011reinforcement}. Actor-critic methods seek to maximise the mean expected payoff of a policy $\pi_\theta$ given by $\mathcal{J}_\theta=\int_{S} \rho^\pi(s)\int_{\mathbf{U}}\pi_\theta(\mathbf{u|s})Q^{\pi}(s,\mathbf{u}) d\mathbf{u}ds$ using gradient ascent on a suitable class of stochastic policies parametrised by $\theta$, where $\rho^\pi(s)$ is the stationary distribution over the states. Updating the policy parameters in the direction of the gradient leads to policy improvement. The gradient of the above objective is $\nabla \mathcal{J}_\theta = \int_{S} \rho^\pi(s)\int_{\mathbf{U}}\nabla\pi_\theta(\mathbf{u|s})Q^{\pi}(s,\mathbf{u}) d\mathbf{u}ds$ \cite{sutton2000policy}. An approximate action-value function based critic $Q^\phi$ is used when estimating the gradient as we do not have access to the true $Q$-function. Since the critic is learnt using finite number of samples, it may deviate from the true $Q$-function, potentially causing incorrect policy updates; this is called the \textit{lagging critic} problem. The problem is exacerbated in multi-agent setting where state-action spaces are very large. 
\paragraph{Tensor Decomposition}
Tensors are high dimensional analogues of matrices and tensor methods generalize matrix algebraic operations to higher orders. Tensor decomposition, in particular, generalizes the concept of low-rank matrix factorization. 
In the rest of this paper, we use $\hat{\cdot}$ to represent tensors.
Formally, an order $n$ tensor $\hat{T}$ has $n$ index sets ${I}_j,\forall j \in\{1..n\}$ and has elements $T(e), \forall e \in \times_{\mathcal{I}} {I}_j $ taking values in a given set $\mathcal{S}$, where $\times$ is the set cross product and we denote the set of index sets by $\mathcal{I}$. Each dimension $\{1..n\}$ is also called a mode.
\begin{figure}[h]
\centering
\includegraphics[width=0.8\linewidth]{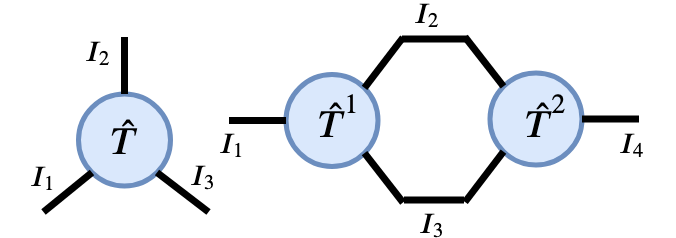}
\caption{Left: Tensor diagram for an order $3$ tensor $\hat T$. Right: Contraction between $\hat T^1$,$\hat T^2$ on common index sets $I_2,I_3$. \label{fig:tdia}}
\end{figure}
An elegant way of representing tensors and associated operations is via tensor diagrams as shown in \cref{fig:tdia}. 
Tensor contraction generalizes the concept of matrix with matrix multiplication. 
For any two tensors $\hat{T}^1$ and $\hat{T}^2$ with $\mathcal{I}_{\cap} = \mathcal{I}^1 \cap \mathcal{I}^2$ we define the contraction operation as $\hat T= \hat{T}^1{\odot} \hat{T}^2$ with $ \hat{T}(e_1,e_2) = \sum_{e\in \times_{\mathcal{I}_{\cap}} {I}_j } \hat T^1(e_1,e)\cdot\hat T^2(e_2,e), e_i \in \times_{\mathcal{I}^i\setminus \mathcal{I}_{\cap}} {I}_j$. 
The contraction operation is associative and can be extended to an arbitrary number of tensors. 
Using this building block, we can define tensor decompositions, which factorizes a (low-rank) tensor in a compact form. 
This can be done with various decompositions~\cite{kolda2009tensor}, such as Tucker, Tensor-Train (also known as Matrix-Product-State), or CP (for Canonical-Polyadic). In this paper, we focus on the latter, which we briefly introduce here.
Just as a matrix can be factored as a sum of rank-$1$ matrices (each being an outer product of vectors),  a tensor can be factored as a sum of rank-1 tensors, the latter being an outer product of vectors. The number of vectors in the outer product is equal to the rank of the tensor, and the number of terms in the sum is called the \emph{rank of the decomposition} (sometimes also called CP-rank). 
Formally, a tensor $\hat T$ can be factored using a (rank--$k$) CP decomposition into a sum of $k$ vector outer products (denoted by $\otimes$), as, 
\begin{align}
\label{CPD}
\hat T=\sum_{r=1}^k w_r\otimes^n u_r^i ,i \in \{1..n\},||u_r^i||_2 =1.
\end{align}	
\section{Methodology}
\label{method}
\subsection{Tensorised Bellman equation}

In this section, we provide the basic framework for Tesseract. We focus here on the discrete action space. The extension for continuous actions is similar and is deferred to \cref{app:cenv} for clarity of exposition.
\begin{proposition}
Any real-valued function $f$ of $n$ arguments $(x_1..x_n)$ each taking values in a finite set $x_i\in \mathcal{D}_i$ can be represented as a tensor $\hat f$ with modes corresponding to the domain sets $\mathcal{D}_i$ and entries $\hat f(x_1..x_n) = f(x_1..x_n)$.  
\end{proposition}
Given a multi-agent problem $G=\left\langle S,U,P,r,Z,O,n,\gamma\right\rangle$, let $\mathcal{Q} \triangleq \{Q: S\times U^n\to \mathbb{R}\}$ be the set of real-valued functions on the state-action space. We are interested in the \emph{curried }\cite{barendregt1984introduction} form $Q: S\to U^n\to \mathbb{R},Q\in \mathcal{Q}$ so that $Q(s)$ is an order $n$ tensor (We use functions and tensors interchangeably where it is clear from context). Algorithms in Tesseract operate directly on the curried form and preserve the structure implicit in the output tensor. (Currying in the context of tensors implies fixing the value of some index. Thus, Tesseract-based methods keep action indices free and fix only state-dependent indices.)

We are now ready to present the tensorised form of the Bellman equation shown in \cref{eq:bell}. \cref{fig:tbell} gives the equation where $\hat I$ is the identity tensor of size $|S|\times|S|\times|S|$. The dependence of the action-value tensor $\hat Q^\pi$ and the policy tensor $\hat U^\pi$ on the policy is denoted by superscripts $\pi$. The novel \textbf{Tensorised Bellman equation} provides a theoretically justified foundation for the approximation of the joint $Q$-function, and the subsequent analysis (Theorems 1-3) for learning using this approximation.

\begin{figure}[h]
\centering
\includegraphics[width=\linewidth]{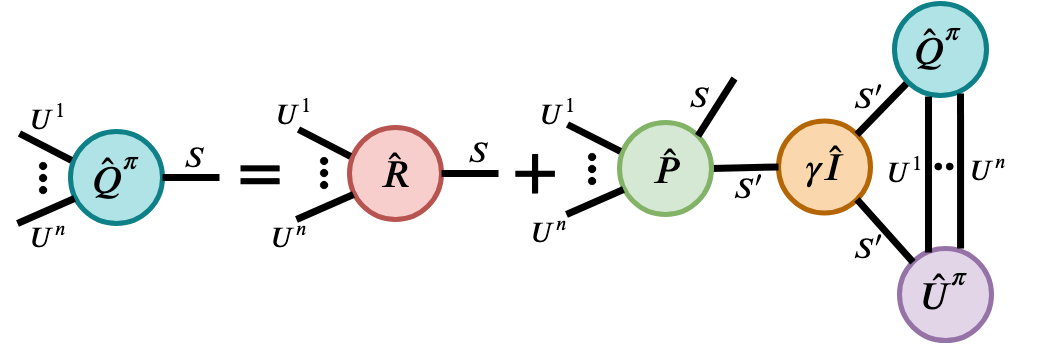}
\caption{\textbf{Tensorised Bellman Equation} for $n$ agents. There is an edge for each agent $i \in \mathcal{A}$ in the corresponding nodes $\hat Q^\pi,\hat U^\pi, \hat R, \hat P$ with the index set $U^i$.\label{fig:tbell}}
\end{figure}

\subsection{\textsc{Tesseract} Algorithms}
\label{subsec:TessAlgos}
\captionsetup[algorithm]{format=hang,singlelinecheck=false}


For any $k \in \mathbb{N}$ let $\mathcal{Q}_k \triangleq \{Q: Q\in \mathcal{Q} \land rank(Q(\cdot, s)) \leq k, \forall s \in S\}$. Given any policy $\pi$ we are interested in projecting $Q^\pi$ to $\mathcal{Q}_k$ using the projection operator $\Pi_k(\cdot) = \argmin_{Q \in \mathcal{Q}_k} ||\cdot-Q||_{\pi,F}$. where $||X||_{\pi,F} \triangleq \mathbb{E}_{s\sim\rho^\pi(s)}[||X(s)||_{F}]$ is the weighted Frobenius norm w.r.t.\ policy visitation over states. Thus a simple planning based algorithm for rank $k$ \textsc{Tesseract} would involve starting with an arbitrary $Q_0$ and successively applying the Bellman operator $\mathcal{T}^{\pi}$ and the projection operator $\Pi_k$ so that $Q_{t+1} = \Pi_k\mathcal{T}^{\pi}Q_t$. 

As we show in \cref{rankbQ}, constraining the underlying tensors for dynamics and rewards ($\hat P, \hat R$) is sufficient to bound the CP-rank of $\hat Q$. From this insight, we propose a model-based RL version for \textsc{Tesseract} in \cref{alg:model-based}. The algorithm proceeds by estimating the underlying MDP dynamics using the sampled trajectories obtained by executing the behaviour policy $\pi = (\pi^i)_1^n$ (factorisable across agents) satisfying \cref{thm:debound}. Specifically, we use a rank $k$ approximate CP-Decomposition to calculate the model dynamics $R, P$ as we show in \cref{analysis}. Next $\pi$ is evaluated using the estimated dynamics, which is followed by policy improvement, \cref{alg:model-based} gives the pseudocode for the model-based setting. The termination and policy improvement decisions in \cref{alg:model-based} admit a wide range of choices used in practice in the RL community. Example choices for internal iterations which broadly fall under approximate policy iteration include: 1) Fixing the number of applications of Bellman operator 2) Using norm of difference between consecutive Q estimates etc., similarly for policy improvement several options can be used like $\epsilon$-greedy (for Q derived policy), policy gradients (parametrized policy)~\cite{sutton2011reinforcement}

\begin{algorithm}[h!]
	\caption{Model-based Tesseract\label{alg:model-based}}
	\begin{algorithmic}[1]
	    \STATE \mbox{Initialise rank $k$, $\pi = (\pi^i)_1^n$  and $\hat Q$: \cref{thm:debound}}
		\STATE \mbox{Initialise model parameters  $\hat P,\hat R$}
		\STATE Learning rate $\leftarrow \alpha$,$\mathcal{D} \leftarrow \left\{ \right\}$ 
		\FOR{each episodic iteration i}
		\STATE Do episode rollout $\tau_i = \left\{(s_t,\mathbf{u}_t,r_t,s_{t+1})_{0}^L \right\}$ using $\pi$
		\STATE $\mathcal{D} \leftarrow \mathcal{D}\cup\left\{\tau_i \right\}$
		\STATE Update $\hat P,\hat R$ using CP-Decomposition on moments from $\mathcal{D}$ (\cref{thm:debound})
		\FOR{each internal iteration j}
		\STATE $\hat Q \gets \mathcal{T}^\pi \hat Q$
		\ENDFOR
		\STATE Improve $\pi$ using $\hat Q$
		\ENDFOR
		\STATE Return $\pi, \hat Q$
	\end{algorithmic}
\end{algorithm}

For large state spaces where storage and planning using model parameters is computationally difficult (they are $\mathcal{O}(kn|U||S|^2)$ in number), 
we propose a model-free approach using a deep network where the rank constraint on the $Q$-function is directly embedded into the network architecture. \cref{fig:tnet} gives the general network architecture for this approach and \cref{alg:model-free} the associated pseudo-code. Each agent in \cref{fig:tnet} has a policy network parameterized by $\theta$ which is used to take actions in a decentralised manner. 
\begin{figure}
    \centering
    \includegraphics[width=\linewidth]{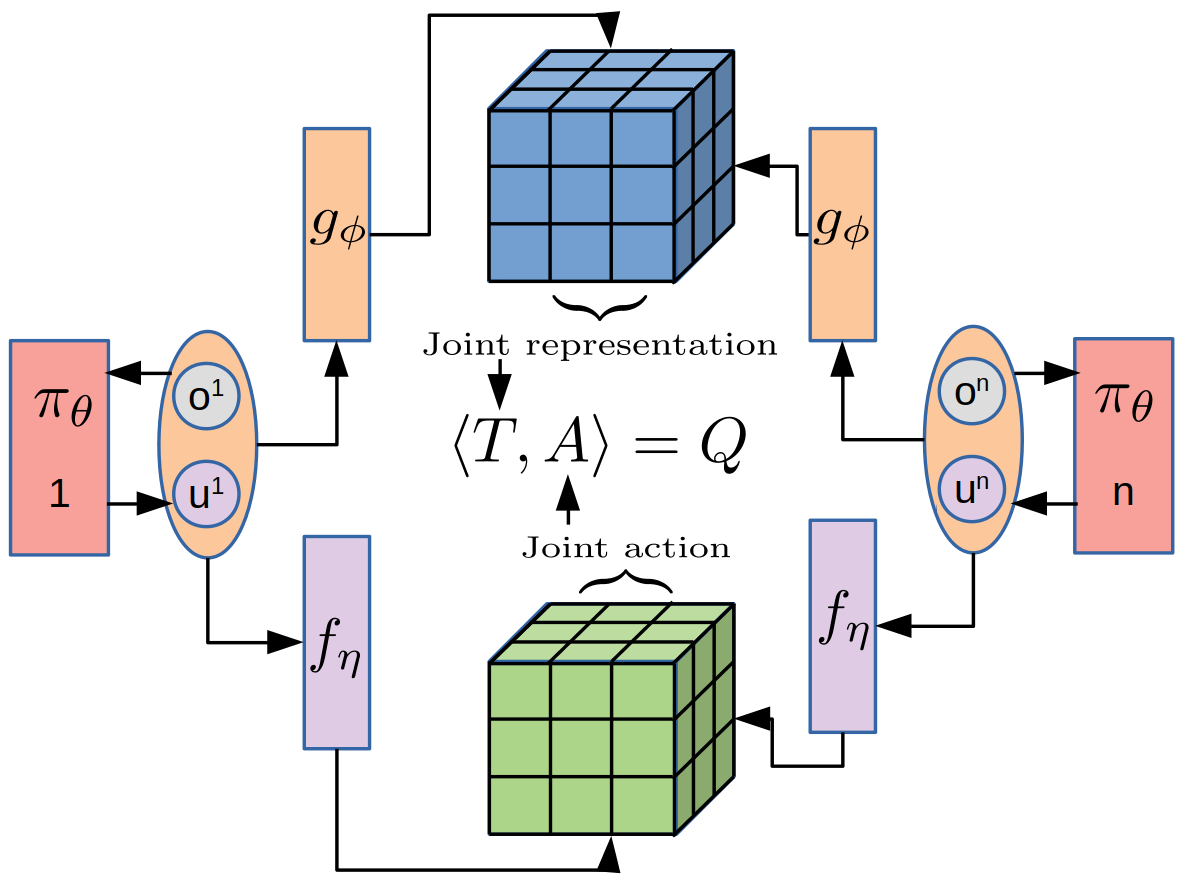}
    \caption{Tesseract architecture \label{fig:tnet}}
\end{figure}   
The observations of the individual agents along with the actions are fed through representation function $g_\phi$ whose output is a set of $k$ unit vectors of dimensionality $|U|$ corresponding to each rank. 
The output $g_{\phi,r}(s^i)$ corresponding to each agent $i$ for factor $r$ can be seen as an action-wise contribution to the joint utility from the agent corresponding to that factor. The joint utility here is a product over individual agent utilities. For partially observable settings, an additional RNN layer can be used to summarise agent trajectories. The joint action-value estimate of the tensor $\hat Q(s)$ by the centralized critic is: 
\begin{align}
\label{eq:cpa}
\hat Q(s) \approx T = \sum_{r=1}^k w_r\otimes^n g_{\phi,r}(s^i) ,i 
\in \{1..n\}, \vspace{-5pt}
\end{align}
where the weights $w_r$ are learnable parameters exclusive to the centralized learner. In the case of value based methods where the policy is implicitly derived from utilities, the policy parameters $\theta$ are merged with $\phi$. The network architecture is agnostic to the type of the action space (discrete/continuous) and the action-value corresponding to a particular joint-action $(u^1..u^n)$ is the inner product $\langle T, A \rangle$ where $A = \otimes^n u^i$ (This reduces to indexing using joint action in \cref{eq:cpa} for discrete spaces). More representational capacity can be added to the network by creating an abstract representation for actions using $f_\eta$, which can be any arbitrary monotonic function (parametrised by $\eta$) of vector output of size $m \geq |U|$ and preserves relative order of utilities across actions; this ensures that the optimal policy is learnt as long as it belongs to the hypothesis space. 
In this case $A = \otimes^n f_\eta(u^i)$ and the agents also carry a copy of $f_\eta$ during the execution phase. Furthermore, the inner product $\langle T, A \rangle$ can be computed efficiently using the property $$\langle T, A \rangle = \sum_{r=1}^k w_r\prod_1^n \langle f_\eta(u^i)g_{\phi,r}(s^i) \rangle ,i \in \{1..n\}$$ which is $O(nkm)$ whereas a naive approach involving computation of the tensors first would be $O(km^n)$. Training the Tesseract-based $Q$-network involves minimising the squared TD loss \cite{sutton2011reinforcement}:
\begin{align}
\mathcal{L}_{TD}(\phi,\eta) = \mathbb{E}_{\pi}[(&Q(\mathbf{u}_t,s_t; \phi,\eta)\\- &[r(\mathbf{u}_t,s_t)+\gamma Q(\mathbf{u}_{t+1},s_{t+1}; \phi^-,\eta^-)])^2],
\end{align}
where $\phi^-,\eta^-$ are target parameters. Policy updates involve gradient ascent w.r.t.\ to the policy parameters $\theta$ on the objective $\mathcal{J}_\theta=\int_{S} \rho^\pi(s)\int_{\mathbf{U}}\pi_\theta(\mathbf{u|s})Q^{\pi}(s,\mathbf{u}) d\mathbf{u}ds$. More sophisticated targets can be used to reduce the policy gradient variance \citep{greensmith2004variance, zhao2016regularized} and propagate rewards efficiently \citep{sutton1988learning}. Note that 
\cref{alg:model-free} does not require the individual-global maximisation principle \citep{son2019qtran} typically assumed by value-based MARL methods in the CTDE setting, as it is an actor-critic method. 
In general, any form of function approximation and compatible model-free approach can be interleaved with Tesseract by appropriate use of the projection function $\Pi_k$. 

\begin{algorithm}[h!]
    \caption{Model-free Tesseract\label{alg:model-free}}
    \begin{algorithmic}[1]
	    \STATE \mbox{Initialise rank $k$, parameter vectors $\theta, \phi, \eta$}
		\STATE Learning rate $\leftarrow \alpha$,$\mathcal{D} \leftarrow \left\{ \right\}$ 
		\FOR{each episodic iteration i}
		\STATE Do episode rollout $\tau_i = \left\{(s_t,\mathbf{u}_t,r_t,s_{t+1})_{0}^L \right\}$ using $\pi_\theta$
		\STATE $\mathcal{D} \leftarrow \mathcal{D}\cup\left\{\tau_i \right\}$
		\STATE Sample batch $\mathcal{B} \subseteq \mathcal{D}$.
		\STATE Compute empirical estimates for $\mathcal{L}_{TD}, \mathcal{J}_\theta$
		\STATE $\phi \leftarrow \phi - \alpha \nabla_\phi \mathcal{L}_{TD}$ (Rank $k$ projection step)
		\STATE $\eta \leftarrow \eta - \alpha \nabla_\eta \mathcal{L}_{TD}$ (Action representation update)
		\STATE $\theta \leftarrow \theta + \alpha \nabla_\theta \mathcal{J}_\theta$ (Policy update)
		\ENDFOR
		\STATE Return $\pi, \hat Q$
	\end{algorithmic}
\end{algorithm}



\subsection{Why Tesseract?}
\label{robo_nav_sec}
As discussed in \cref{intro}, $Q(s)$ is an object of prime interest in MARL. Value based methods  \cite{sunehag_value-decomposition_2017, rashid2018qmix, yao2019smix} that directly approximate the optimal action values $Q^*$ place constraints on $Q(s)$ such that it is a monotonic combination of agent utilities. In terms of Tesseract this directly translates to finding the best projection constraining $Q(s)$ to be rank one (\cref{discussion:vdn}). Similarly, the following result demonstrates containment of action-value functions representable by FQL\citep{chen2018factorized} which uses a learnt inner product to model pairwise agent interactions (\textbf{proof and additional results in \cref{discussion:vdn}}):.
\begin{proposition}
\label{prop:fql}
The set of joint Q-functions representable by FQL is a subset of that representable by \textsc{Tesseract}.
\end{proposition}

MAVEN \cite{mahajan2019maven} illustrates how rank $1$ projections can lead to insufficient exploration and provides a method to avoid suboptimality by using mutual information (MI) to learn a diverse set of rank $1$ projections that correspond to different joint behaviours. In Tesseract, this can simply be achieved by finding the best approximation constraining $Q(s)$ to be rank $k$. Moreover, the CP-decomposition problem, being a product form (\cref{CPD}), is well posed, whereas in \cite{mahajan2019maven} the problem form  is $\hat T=\sum_{r=1}^k w_r\oplus^n u_r^i ,i \in \{1..n\},||u_r^i||_2 =1$, which requires careful balancing of different factors $\{1..k\}$ using MI as otherwise all factors collapse to the same estimate. The above improvements are equally important for the critic in actor-critic frameworks. Note that \textsc{Tesseract} is complete
in the sense that every possible joint Q-function is representable by it given sufficient approximation rank. This follows as every possible Q-tensor can be expressed as linear combination of one-hot tensors (which form a basis for the set).

Many real world problems have high-dimensional observation spaces that are encapsulated in an underlying low dimensional latent space that governs the transition and reward dynamics \cite{azizzadenesheli2016reinforcement}. 
For example, in the case of robot navigation, the observation is high dimensional visual and sensory input but solving the underlying problem requires only knowing the 2D position. 
Standard RL algorithms that do not address modelling the latent structure in such problems typically incur poor performance and intractability. In \cref{analysis} we show how Tesseract can be leveraged for such scenarios. 
Finally, projection to a low rank offers a natural way of regularising the  approximate $Q$-functions and makes them easier to learn, which is important for making value function approximation amenable to multi-agent settings. Specifically for the case of actor-critic methods, this provides a natural way to make the critic learn more quickly. Additional discussion about using Tesseract for continuous action spaces can be found in \cref{app:cenv}.

\section{Analysis}

\label{analysis}
In this section we provide a PAC analysis of model-based Tesseract (\cref{alg:model-based}). We focus on the MMDP setting (\cref{background}) for the simplicity of notation and exposition; guidelines for other settings are provided in \cref{app:proofs}.  

The objective of the analysis is twofold: Firstly it provides concrete quantification of the sample efficiency gained by model-based policy evaluation. Secondly, it provides insights into how Tesseract can similarly reduce sample complexity for model-free methods. \textbf{Proofs for the results stated can be found in \cref{app:proofs}}. We begin with the assumptions used for the analysis:

\begin{assumption}
\label{a1}
For the given MMDP $G=\left\langle S,U,P,r,n,\gamma\right\rangle$, the reward tensor $\hat R(s),\forall s\in S$ has bounded rank $k_1\in \mathbb{N}$. 
\end{assumption}

Intuitively, a small $k_1$ in \cref{a1} implies that the reward is dependent only on a small number of intrinsic factors characterising the actions.

\begin{assumption}
\label{a2}
For the given MMDP $G=\left\langle S,U,P,r,n,\gamma\right\rangle$, the transition tensor $\hat P(s,s'),\forall s,s'\in S$ has bounded rank $k_2\in \mathbb{N}$.
\end{assumption}

Intuitively a small $k_2$ in \cref{a2} implies that only a small number of intrinsic factors characterising the actions lead to meaningful change in the joint state.
\cref{a1}{-2} always hold for a finite MMDP as CP-rank is upper bounded by $\Pi_{j=1}^n |U_j|$, where $U_j$ are the action sets.
\begin{assumption}
\label{a3}
The underlying MMDP is ergodic for any policy $\pi$ so that there is a stationary distribution $\rho^\pi$.
\end{assumption}

Next, we define coherence parameters, which are quantities of interest for our theoretical results: for reward decomposition $\hat R(s) = \sum_r w_{r,s}\otimes^n v_{r,i,s}$, let $\mu_{s} = \sqrt{n}\max_{i,r,j}|v_{r,i,s}(j)|$, $w_{s}^{\text{max}} = \max_{i,r}w_{r,s}$,$w_{s}^{\text{min}} = \min_{i,r}w_{r,s}$. Similarly define the corresponding quantities for $\mu_{s,s'}, w_{s,s'}^{\text{max}},w_{s,s'}^{\text{min}}$ for transition tensors $\hat P(s,s')$. A low coherence implies that the tensor's mass is evenly spread and helps bound the possibility of never seeing an entry with very high mass (large absolute value of an entry).

\begin{theorem}
\label{rankbQ}
For a finite MMDP the action-value tensor satisfies $rank(\hat Q^\pi(s))\leq k_1+k_2|S|,\forall s \in S, \forall \pi$.
\end{theorem}
\begin{proof}
We first unroll the Tensor Bellman equation in \cref{fig:tbell}. The first term $\hat R$ has bounded rank $k_1$ by \cref{a1}. Next, each contraction term on the RHS is a linear combination of $\{\hat P(s,s')\}_{s'\in S}$ each of which has bounded rank $k_2$ (\cref{a2}). The result follows from the sub-additivity of CP-rank.
\end{proof}

\cref{rankbQ} implies that for approximations with enough factors, policy evaluation converges:

\begin{corollary}
\label{cor:suf_rank}
For all $k\geq k_1+k_2|S|$, the procedure $Q_{t+1}\leftarrow \Pi_{k}\mathcal{T}^{\pi}Q_{t}$ converges to $Q^\pi$ for all $Q_0,\pi$.
\end{corollary}

\cref{cor:suf_rank} is especially useful for the case of M-POMDP and M-ROMDP with $|Z| >> |S|$, i.e., where the intrinsic state space dimensionality is small in comparison to the dimensionality of the observations (like robot navigation \cref{robo_nav_sec}). In these cases the Tensorised Bellman equation \cref{fig:tbell} can be augmented by padding the transition tensor $\hat P$ with the observation matrix and the lower bound in \cref{cor:suf_rank} can be improved using the intrinsic state dimensionality.

We next give a PAC result on the number of samples required to infer the reward and state transition dynamics for finite MDPs with high probability using sufficient approximate rank $k \geq k_1, k_2$: 

\begin{theorem}[Model based estimation of $\hat R, \hat P$ error bounds]
\label{thm:debound}
Given any $\epsilon>0, 1>\delta>0$, for a policy $\pi$ with the policy tensor satisfying $\pi(\mathbf{u}|s)\geq \Delta$, where
\begin{align}
\nonumber
\Delta = \max_s 
\frac{C_1 \mu_{s}^6 k^5 (w_{s}^{\text{max}})^4 \log(|U|)^4 \log(3k||R(s)||_{F}/\epsilon)}
{|U|^{n/2} (w_{s}^{\text{min}})^4}  
\end{align}
and $C_1$ is a problem dependent positive constant. There exists $N_0$ which is $O(|U|^{\frac{n}{2}})$ and polynomial in $\frac{1}{\delta},\frac{1}{\epsilon}, k$ and relevant spectral properties of the underlying MDP dynamics such that for samples $\geq N_0$, we can compute the estimates $\bar R(s), \bar P(s,s')$ such that w.p. $\geq 1-\delta$, $||\bar{R}(s)-\hat R(s)||_F\leq \epsilon, ||\bar{P}(s,s')-\hat P(s,s')||_F\leq \epsilon, \forall s,s' \in S$.
\end{theorem}

\cref{thm:debound} gives the relation between the order of the number of samples required to estimate dynamics and the tolerance for approximation. \cref{thm:debound} states that aside from allowing efficient PAC learning of the reward and transition dynamics of the multi-agent MDP, \cref{alg:model-based} requires only $O(|U|^{\frac{n}{2}})$ to do so, which is a vanishing fraction of $|U|^n$, the total number of joint actions in any given state. This also hints at why a tensor based approximation of the $Q$-function helps with sample efficiency. Methods that do not use the tensor structure typically use $O(|U|^n)$ samples. The bound is also useful for off-policy scenarios, where only the behaviour policy needs to satisfy the bound.
Given the result in \cref{thm:debound}, it is natural to ask what is the error associated with computing the action-values of a policy using the estimated transition and reward dynamics. We address this in our next result, but first we present a lemma bounding the total variation distance between the estimated and true transition distributions: 
\begin{lemma}
\label{tvbound}
For transition tensor estimates satisfying $||\bar{P}(s,s')-\hat P(s,s')||_F\leq \epsilon$, we have for any given state-action pair $(s,a)$, the distribution over the next states follows: $TV(P'(\cdot|s,a),P(\cdot|s,a))\leq \frac{1}{2}(|1-f|+f|S|\epsilon)$ where $\frac{1}{1+\epsilon|S|}\leq f \leq\frac{1}{1-\epsilon|S|}$, where $TV$ is the \textit{total variation} distance. Similarly for any policy $\pi$, $TV(\bar P_{\pi}(\cdot|s),P_{\pi}(\cdot|s)), TV(\bar P_{\pi}(s',a'|s),P_{\pi}(s',a'|s))\leq \frac{1}{2}(|1-f|+f|S|\epsilon)$ 
\end{lemma}

We now bound the error of model-based evaluation using approximate dynamics in \cref{thm:q_err}. The first component on the RHS of the upper bound comes from the tensor analysis of the transition dynamics, whereas the second component can be attributed to error propagation for the rewards.
\begin{theorem}[Error bound on policy evaluation]
\label{thm:q_err}
Given a behaviour policy $\pi_b$ satisfying the conditions in \cref{thm:debound} and executed for steps $\geq N_0$, for any policy $\pi$ the model based policy evaluation $Q_{\bar P,\bar R}^\pi$ satisfies:
\begin{align}
|Q_{P,R}^\pi(s,a) - Q_{\bar P,\bar R}^\pi(s,a)|\leq &(|1-f|+f|S|\epsilon)\frac{\gamma}{2(1-\gamma)^2} \\ &+ \frac{\epsilon}{1-\gamma}, \forall (s,a)\in S\times U^n
\end{align} where $f$ is as defined in \cref{tvbound}.
\end{theorem}
Additional theoretical discussion can be found in \cref{app:atd}


 \section{Experiments}
\label{sec:exps}

In this section we present the empirical results on the StarCraft domain. Experiments for a more didactic domain of Tensor games can be found in \cref{app:tg}. We use the model-free version of \textsc{Tesseract} (\cref{alg:model-free}) for all the experiments.

\begin{figure*}[h]
	\centering
	\subfigure[3s5z \textbf{Easy}]{
		\includegraphics[width=0.325\linewidth]{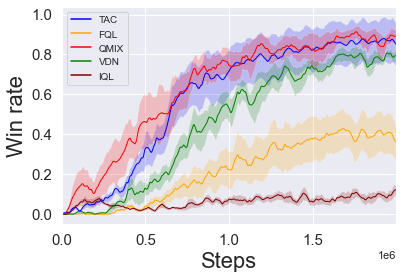}\label{fig:3s5z_smac}}
	\subfigure[2s\_vs\_1sc \textbf{Easy}]{
		\includegraphics[width=0.325\linewidth]{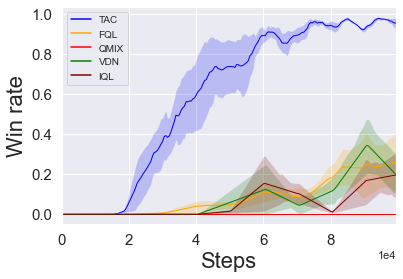}\label{fig:2s_vs_1sc}}
	\subfigure[2c\_vs\_64zg \textbf{Hard}]{
		\includegraphics[width=0.325\linewidth]{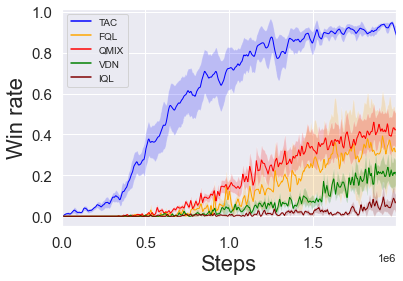}\label{fig:2c_vs_64z}}
	\subfigure[5m\_vs\_6m \textbf{Hard}]{
		\includegraphics[width=0.325\linewidth]{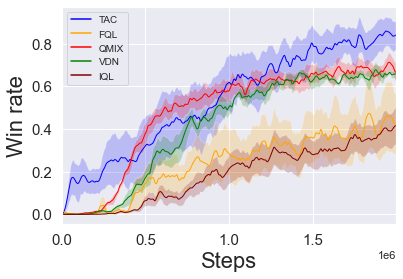}\label{fig:5m_vs_6m}}
	\subfigure[MMM2 \textbf{Super Hard}]{
		\includegraphics[width=0.325\linewidth]{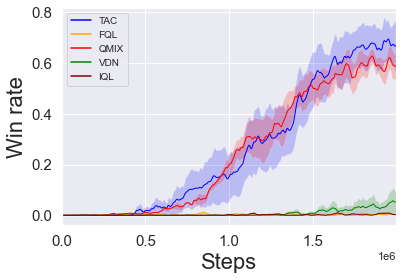}\label{fig:MMM2}}
	\subfigure[27m\_vs\_30m \textbf{Super Hard}]{
		\includegraphics[width=0.325\linewidth]{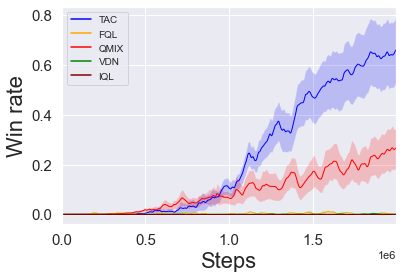}\label{fig:27m_vs_30m}}
	\subfigure[6h\_vs\_8z \textbf{Super Hard}]{
		\includegraphics[width=0.325\linewidth]{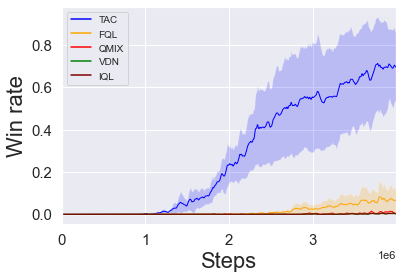}\label{fig:6h8z}}
	\subfigure[Corridor \textbf{Super Hard}]{
		\includegraphics[width=0.325\linewidth]{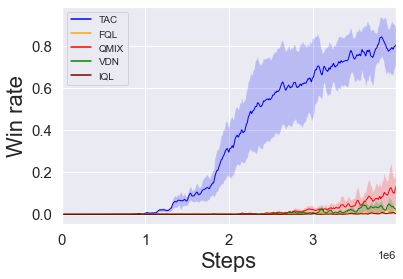}\label{fig:corridor}}
	\caption{Performance of different algorithms on different SMAC scenarios: \textcolor{blue}{TAC}, \textcolor{red}{QMIX}, \textcolor[rgb]{0,0.7,0}{VDN}, \textcolor{orange}{FQL}, \textcolor[rgb]{0.76, 0.13, 0.28}{IQL}. \label{fig:smac_exp}}
\end{figure*}

\paragraph{StarCraft II}
We consider a challenging set of cooperative scenarios from the StarCraft  Multi-Agent Challenge (SMAC) \citep{samvelyan2019starcraft}. Scenarios in SMAC have been classified as \textbf{Easy, Hard and Super-hard} according to the performance of exiting algorithms on them. We compare \textsc{Tesseract} (\textcolor{blue}{TAC} in plots) to, \textcolor{red}{QMIX} \citep{rashid2018qmix}, \textcolor[rgb]{0,0.7,0}{VDN} \citep{sunehag_value-decomposition_2017}, \textcolor{orange}{FQL} \citep{chen2018factorized}, and \textcolor[rgb]{0.76, 0.13, 0.28}{IQL} \citep{tan_multi-agent_1993}. VDN and QMIX use monotonic approximations for learning the Q-function. FQL uses a pairwise factorized model to capture effects of agent interactions in joint Q-function, this is done by learning an inner product space for summarising agent trajectories. IQL ignores the multi-agentness of the problem and learns an independent per agent policy for the resulting non-stationary problem. \cref{fig:smac_exp} gives the win rate of the different algorithms averaged across five random runs. \cref{fig:2c_vs_64z} features 2c\_vs\_64zg, a hard scenario that contains two
allied agents but 64 enemy units (the largest in the SMAC domain) making the action space of
the agents much larger than in the other scenarios. \textsc{Tesseract} gains a huge lead over all the other algorithms in just one million steps. For the asymmetric scenario of 5m\_vs\_6m \cref{fig:5m_vs_6m}, \textsc{Tesseract}, QMIX, and VDN learn effective policies, similar behavior occurs in the heterogeneous scenarios of 3s5z \cref{fig:3s5z_smac} and MMM2\cref{fig:MMM2} with the exception of VDN for the latter. In 2s\_vs\_1sc in \cref{fig:2s_vs_1sc}, which requires a `kiting' strategy to defeat the spine crawler, \textsc{Tesseract} learns an optimal policy in just 100k steps. In the \textbf{super-hard} scenario of 27m\_vs\_30m \cref{fig:27m_vs_30m} having largest ally team of 27 marines, \textsc{Tesseract} again shows improved sample efficiency; this map also shows \textsc{Tesseract}'s ability to scale with the number of agents. Finally in the \textbf{super-hard} scenarios of 6 hydralisks vs 8 zealots \cref{fig:6h8z} and Corridor \cref{fig:corridor} which require careful exploration, \textsc{Tesseract} is the only algorithm which is able to find a good policy. We observe that IQL doesn't perform well on any of the maps as it doesn't model agent interactions/non-stationarity explicitly. FQL loses performance possibly because modelling just pairwise interactions with a single dot product might not be expressive enough for joint-Q. Finally, VDN and QMIX are unable to perform well on many of the challenging scenarios possibly due to the monotonic approximation affecting the exploration adversely \citep{mahajan2019maven}.
Additional plots and experiment details can be found in \cref{app:sc2} with \textbf{comparison with other baselines in \cref{app:additional_sc2}} including QPLEX\citep{wang2020qplex}, QTRAN\citep{ son2019qtran}, HQL\citep{matignon2007hysteretic}, COMA\citep{foerster2018counterfactual} . We detail the techniques used for stabilising the learning of tensor decomposed critic in \cref{app:techniques}.
	
\section{Related Work}
Previous methods for modelling multi-agent interactions include those that use coordination graph methods for learning a factored joint action-value estimation \cite{guestrin2002coordinated,guestrin2002context,bargiacchi2018learning}, however typically require knowledge of the underlying coordination graph. 
To handle the exponentially growing complexity of the joint action-value functions with the number of agents, a series of value-based methods have explored different forms of value function factorisation.
VDN~\cite{sunehag_value-decomposition_2017} and QMIX~\cite{rashid2018qmix} use monotonic approximation with latter using a mixing network conditioned on global state. 
QTRAN~\cite{son2019qtran} avoids the weight constraints imposed by QMIX by formulating multi-agent learning as an optimisation problem with linear constraints and relaxing it with L2 penalties. 
MAVEN~\cite{mahajan2019maven} learns a diverse ensemble of monotonic approximations by conditioning agent $Q$-functions on a latent space which helps overcome the detrimental effects of QMIX’s monotonicity constraint on exploration. Similarly, Uneven~\cite{gupta2020uneven} uses universal successor features for efficient exploration in the joint action space. 
Qatten~\cite{Yang2020QattenAG} makes use of a multi-head attention mechanism to decompose $Q_{tot}$ into a linear combination of per-agent terms. RODE~\cite{wang2020rode} learns an action effect based role decomposition for sample efficient learning.
Policy gradient methods, on the other hand, often utilise the actor-critic framework to cope with decentralisation.
MADDPG~\cite{lowe2017multi} trains a centralised critic for each agent. 
COMA~\cite{foerster2018counterfactual} 
employs a centralised critic and a counterfactual advantage function.
These actor-critic methods, however, suffer from poor sample efficiency compared to value-based methods and often converge to sub-optimal local minima. While sample efficiency has been an important goal for single agent reinforcement learning methods ~\cite{mahajan2017symmetryde, mahajan2017symmetryl, kakade2003sample, lattimore2013sample}, in this work we shed light on attaining sample efficiency for cooperative multi-agent systems using low rank tensor approximation.

\vspace{-2mm}
\emph{Tensor methods} have been used in machine learning, in the context of learning latent variable models~\cite{anandkumar2014tensor} and signal processing \cite{sidiropoulos2017tensor}. 
Tensor methods provides powerful analytical tools that have been used for various applications, including the theoretical analysis of deep neural networks~\cite{cohen2016expressive}.
Model compression using tensors~\cite{cheng2017survey} has recently gained momentum owing to the large sizes of deep neural nets.
Using tensor decomposition within deep networks, it is possible to both compress and speed them up~\cite{cichocki2017tensor,t_net}. They allow generalization to higher orders ~\cite{kossaifi2019efficient} and have also been used for multi-task learning and domain adaptation~\cite{bulat2019incremental}. In contrast to prior work on value function factorisation, \textsc{Tesseract} provides a natural spectrum for approximation of action-values based on the rank of approximation and provides theoretical guarantees derived from tensor analysis. Multi-view methods utilising tensor decomposition have previously been used in the context of partially observable single-agent RL~\cite{azizzadenesheli2016reinforcement,azizzadenesheli2019reinforcement}. There the goal is to efficiently infer the underlying MDP parameters for planning under rich observation settings~\cite{krishnamurthy2016pac}. Similarly \citep{bromuri2012tensor} use four dimensional factorization to generalise across Q-tables whereas here we use them for modelling interactions across multiple agents.
\section{Conclusions \& Future Work}
We introduced \textsc{Tesseract}, a novel framework utilising the insight that the joint action value function for MARL can be seen as a tensor. \textsc{Tesseract} provides a means for developing new sample efficient algorithms and obtain essential guarantees about convergence and recovery of the underlying dynamics. We further showed novel PAC bounds for learning under the framework using model-based algorithms. We also provided a model-free approach to implicitly induce low rank tensor approximation for better sample efficiency and showed that it outperforms current state of art methods. 
There are several interesting open questions to address in future work, such as convergence and error analysis for rank insufficient approximation, and analysis of the learning framework under different types of tensor decompositions like Tucker and tensor-train \citep{kolda2009tensor}.

\section{Acknowledgements}
AM is funded by the J.P. Morgan A.I. fellowship. Part of this work was done during AM's internship at NVIDIA. This project has received funding from the European Research Council under the European Union’s Horizon 2020 research and innovation programme (grant agreement number 637713).  The experiments were made possible by generous equipment grant from NVIDIA.	

\clearpage
\newpage
\bibliographystyle{icml2021}
\bibliography{tesseract}
\clearpage
\newpage

\onecolumn

\appendix
\section{Additional Proofs}
\addtocounter{theorem}{-2}
\addtocounter{proposition}{-1}
\addtocounter{lemma}{-1}
\label{app:proofs}
\subsection{Proof of \cref{thm:debound}}
\label{proof:app_dbound}

\begin{theorem}[Model based estimation of $\hat R, \hat P$ error bounds]
Given any $\epsilon>0, 1>\delta>0$, for a policy $\pi$ with the policy tensor satisfying $\pi(\mathbf{u}|s)\geq \Delta$, where
\begin{align}
\label{eq:polcon}
\Delta = \max_s 
\frac{C_1 \mu_{s}^6 k^5 (w_{s}^{\text{max}})^4 \log(|U|)^4 \log(3k||R(s)||_{F}/\epsilon)}
{|U|^{n/2} (w_{s}^{\text{min}})^4}  
\end{align}
and $C_1$ is a problem dependent positive constant. There exists $N_0$ which is $O(|U|^{\frac{n}{2}})$ and polynomial in $\frac{1}{\delta},\frac{1}{\epsilon}, k$ and relevant spectral properties of the underlying MDP dynamics such that for samples $\geq N_0$, we can compute the estimates $\bar R(s), \bar P(s,s')$ such that w.p. $\geq 1-\delta$, $||\bar{R}(s)-\hat R(s)||_F\leq \epsilon, ||\bar{P}(s,s')-\hat P(s,s')||_F\leq \epsilon, \forall s,s' \in S$.
\end{theorem}
\begin{proof}
For the simplicity of notation and emphasising key points of the proof, we focus on orthogonal symmetric tensors with $n=3$. Guidelines for more general cases are provided by the end of the proof. 

We break the proof into three parts:
Let policy $\pi$ satisfy $\pi(\mathbf{u}|s)\geq \Delta$ \cref{eq:polcon}. Let $\rho$ be the stationary distribution of $\pi$ (exists by \cref{a3}) and let $N_1  = \max_s \frac{1}{\rho(s)}\log\Big(\frac{12\sqrt{k}||R(s)||_F}{\epsilon}\Big)$. From $N_1$ samples drawn from $\rho$ by following $\pi$, we estimate $\bar{R}$, the estimated reward tensor computed by using Algorithm $1$ in \cite{jain2014provable}. We have by application of union bound along with Theorem $1.1$ in \cite{jain2014provable} for each $s\in S$, w.p. $\geq 1 - |U|^{-5}\log_2\Big(\frac{12\sqrt{k}\prod_s||R(s)||_F}{\epsilon}\Big) = p_\epsilon$, $||\bar{R}(s)-\hat R(s)||_F\leq \epsilon/3, \forall s\in S$. We now provide a boosting scheme to increase the confidence in the estimation of $\hat R(\cdot)$ from $p_\epsilon$  to $1 - \delta/3$. Let $\eta = \frac{1}{2}\Big(p_\epsilon -\frac{1}{2}\Big) > 0$ (for clarity of the presentation we assume $p_\epsilon >\frac{1}{2}$ and refer the reader to \cite{kearns1994introduction} for the other more involved case). We compute $M$ independent estimates $\{\bar R_i, i \in \{1..M\}\}$ for $\hat R(s)$ and find the biggest cluster $\mathcal{C} \subseteq \{\bar R_i\}$ amongst the estimates such that for any $\bar R_i, \bar R_j \in \mathcal{C}, ||\bar R_i-\bar R_j||_F \leq \frac{2\epsilon}{3}$. We then output any element of $\mathcal{C}$. Intuitively as $p_\epsilon >\frac{1}{2}$, most of the estimates will be near the actual value $\hat R(s)$, this can be confirmed by using the Hoeffding Lemma\cite{kearns1994introduction}. It follows that for $M\geq \frac{1}{2\eta^2}\ln(\frac{3|S|}{\delta})$ the output of the above procedure satisfies $||\bar{R}(s)-\hat R(s)||_F\leq \epsilon$ w.p. $\geq 1 - \frac{\delta}{3|S|}$ for any particular $s$. Thus $MN_1$ samples from stationary distribution are sufficient to ensure that for all $s \in S$, w.p. $\geq 1 - \delta/3$, $||\bar{R}(s)-\hat R(s)||_F\leq \epsilon$.

Secondly we note that $\hat P (s,s')$ for any $s,s' \in S$ is a tensor whose entries are the parameters of a Bernoulli distribution. Under \cref{a2}, it can be seen as a latent topic model \cite{anandkumar2012method} with $k$ factors, $\hat P(s,s') = \sum_{r=1}^{k} w_{s,s',r}\otimes^n u_{s,s',r} $. Moreover it satisfies the conditions in Theorem $3.1$ \cite{anandkumar2012method} so that $\exists N_2  = \max_{s,s'}\frac{1}{\rho(s)} N_2(s,s')$ where each $N_2(s,s')$ is $\mathcal{O}\Big(\frac{k^{10} |S|^2 \ln^2(3|S|/\delta)}{\delta^2 \epsilon^{'2}}\Big)$ depending on the spectral properties of $\hat P(s,s')$ as given in the theorem and satisfies $||\bar{u_{s,s',r}}-u_{s,s',r}||_2\leq \epsilon'$ on running Algorithm B in \cite{anandkumar2012method} w.p. $\geq 1 - \frac{\delta}{3|S|}$. We pick $\epsilon' = \frac{\epsilon}{7n^2k \mu_{s,s'}^2 (w_{s,s'}^{\text{max}})^2} $ so that $||\bar{P}(s,s')-\hat P(s,s')||_F\leq \epsilon, \forall s,s' \in S$. We filter off the effects of sampling from a particular policy by using lower bound constraint in \cref{eq:polcon} and sampling $\frac{N_2}{\Delta}$ samples.

Finally we account for the fact that there is a delay in attaining the stationary distribution $\rho$ and bound the failure probability of significantly deviating from $\rho$ empirically. Let $\rho' = \min_s \rho(s)$ and $t_{\text{mix},\pi}(x)$ represent the minimum number of samples that need to drawn from the Markov chain formed by fixing policy $\pi$ so that for the state distribution $\rho_{t}(s)$ at time step $t = t_{\text{mix},\pi}(x)$ we have $TV(\rho_{t} - \rho)\leq x$ for any starting state $s\in S$ where $TV(\cdot,\cdot)$ is the total variation distance. We let the policy run for a burn in period of $t'=t_{\text{mix},\pi}(\rho'/4)$. For a sample of $N_3$ state transitions after the burn in period, let $\bar \rho$ represent the empirical state distribution. By applying the Hoeffding lemma for each state, we get: $P(|\bar\rho(s) - \rho_{t'}(s)|\geq \rho'/4)\leq 2\exp\Big(\frac{-N_3\rho^{'2}}{8}\Big)$, so that for $N_3\geq \frac{8}{\rho^{'2}}\ln\Big(\frac{6|S|}{\delta}\Big)$ we have w.p. $\geq 1 - \frac{\delta}{3|S|}$, $|\bar\rho(s) - \rho(s)|< \rho'/2, \forall s \in S$.

Putting everything together we get with $t_{\text{mix},\pi}(\rho'/4) + \max\{2MN_1, \frac{2N_2}{\Delta}, N_3\}$ samples, the underlying reward and probability tensors can be recovered such that w.p. $\geq 1-\delta$, $||\bar{R}(s)-\hat R(s)||_F\leq \epsilon, ||\bar{P}(s,s')-\hat P(s,s')||_F\leq \epsilon, \forall s,s' \in S$.

For extending the proof to the case of non-orthogonal tensors, we refer the reader to use whitening transform as elucidated in \cite{anandkumar2014tensor}. Likewise for asymmetric, higher order ($n>3$) tensors methods shown in \cite{jain2014provable, anandkumar2014tensor, anandkumar2012method} should be used. Finally for the case of M-POMDP and M-ROMDP, the corresponding results for single agent POMDP and ROMDP should be used, as detailed in \cite{azizzadenesheli2019reinforcement, azizzadenesheli2016reinforcement} respectively. 
\end{proof}

\subsection{Proof of \cref{tvbound}}
\label{proof:tvbound}
\begin{lemma}
For transition tensor estimates satisfying $||\bar{P}(s,s')-\hat P(s,s')||_F\leq \epsilon$, we have for any given state and action pair $s,a$, the distribution over the next states follows: $TV(P'(\cdot|s,a),P(\cdot|s,a))\leq \frac{1}{2}(|1-f|+f|S|\epsilon)$ where $\frac{1}{1+\epsilon|S|}\leq f \leq\frac{1}{1-\epsilon|S|}$. Similarly for any policy $\pi$, $TV(\bar P_{\pi}(\cdot|s),P_{\pi}(\cdot|s)), TV(\bar P_{\pi}(s',a'|s),P_{\pi}(s',a'|s))\leq \frac{1}{2}(|1-f|+f|S|\epsilon)$ 
\end{lemma}
\begin{proof}
Let $ \bar P(\cdot|s,a)$ be the next state probability estimates obtained from the tensor estimates. We next normalise them across the next states to get the (estimated)distribution $ P'(\cdot|s,a) = f \bar P(\cdot|s,a)$ where $f = \frac{1}{\sum_{s'}\bar P(s'|s,a)}$. Dropping the conditioning for brevity we have:
\begin{align}
TV(P',P) &= \frac{1}{2}\sum_{s'} |P(s')-f \bar P(s')|\\
&\leq\frac{1}{2}(\sum_{s'} |P(s')-f P(s')| +|f P(s')- \bar P(s')|)\\
&=\frac{1}{2}(|1-f|+f|S|\epsilon) 
\end{align}
The other two results follow using the definition of TV and Fubini's theorem followed by reasoning similar to above.
\end{proof}

\subsection{Proof of \cref{thm:q_err}}
\label{proof:app_q_err}
\begin{theorem}
[Error bound on policy evaluation]
Given a behaviour policy $\pi_b$ satisfying the conditions in \cref{thm:debound} and being executed for steps $\geq N_0$, we have that for any policy $\pi$ the model based policy evaluation $Q_{\bar P,\bar R}^\pi$ satisfies:
\begin{align}
|Q_{P,R}^\pi(s,a) - Q_{\bar P,\bar R}^\pi(s,a)|&\leq (|1-f|+f|S|\epsilon)\frac{\gamma}{2(1-\gamma)^2}+   \frac{\epsilon}{1-\gamma}, \forall (s,a)\in S\times U^n
\end{align} 
where $f$ is as defined in \cref{tvbound}.
\end{theorem}
\begin{proof}
Let $\bar P, \bar R$ be the estimates obtained after running the procedure as described in \cref{thm:debound} with samples corresponding to error $\epsilon$ and confidence $1-\delta$. We will bound the error incurred in estimation of the action-values using $\bar P, \bar R$. We have for any $\pi$ by using triangle inequality 
\begin{align}
\label{tring_q}
|Q_{P,R}^\pi(s,a) - Q_{\bar P,\bar R}^\pi(s,a)|&\leq |Q_{P,R}^\pi(s,a)- Q_{\bar P, R}^\pi(s,a)| + |Q_{\bar P, R}^\pi(s,a) - Q_{\bar P,\bar R}^\pi(s,a)|   
\end{align}
where we use the subscript to denote whether actual or approximate values are used for $P,R$ respectively. We first focus on the first term on the RHS of \cref{tring_q}. Let $R_\pi(s_t) = \sum_{a_t} \pi(a_t|s_t) R(s_t,a_t)$. We use $P_{t,\pi}(\cdot|s) = (P_{\pi}(\cdot|s))^t$ to denote the state distribution after $t$ time steps. Consider a horizon $h$ interleaving $Q$ estimate given by:
\begin{align}
Q_h^\pi(s,a) &= R(s_t,a_t)+\sum_{t=1}^{h-1}\gamma^t\mathbb{E}_{\bar P_{t,\pi}(\cdot|s)}[ R_\pi(s_t)]+ \sum_{t=h}^{\infty}\gamma^t\mathbb{E}_{P_{t-h,\pi}(\cdot|s_h)\cdot \bar P_{h,\pi}(s_h|s) }[ R_\pi(s_t)]
\end{align}
Where $s_0=s, a_0=a$ and the first $h$ steps are unrolled according to $\bar P_\pi$, the rest are done using the true transition $P_\pi$. We have that:
\begin{align}
|Q_{P,R}^\pi(s,a) - Q_{\bar P,\bar R}^\pi(s,a)|&=|Q_0^\pi(s,a) - Q_\infty^\pi(s,a)|\leq \sum_{h=0}^\infty |Q_h^\pi(s,a) - Q_{h+1}^\pi(s,a)|
\end{align}
Each term in the RHS of the above can be independently bounded as : 
\begin{align}
|Q_h^\pi(s,a) - Q_{h+1}^\pi(s,a)|=&\gamma^{h+1}\Big|\mathbb{E}_{\bar P_{h+1,\pi}(s_{h+1}|s)}\Big[\sum_{a_{h+1}}\pi(a_{h+1}|s_{h+1})Q_\infty^\pi(s_{h+1}.a_{h+1})\Big]\\&-\mathbb{E}_{P_\pi\bar P_{h,\pi}(s_{h+1}|s)}\Big[\sum_{a_{h+1}}\pi(a_{h+1}|s_{h+1})Q_\infty^\pi(s_{h+1}.a_{h+1})\Big]\Big|
\end{align}
As the rewards are bounded we get the expression above is $\leq \frac{1}{1-\gamma}\gamma^{h+1}TV(\bar P_{\pi}(s',a'|s),P_{\pi}(s',a'|s))$. Finally using \cref{tvbound} we get $\leq (\frac{1}{2}(|1-f|+f|S|\epsilon))\frac{\gamma^{h+1}}{1-\gamma}$. And plugging in the original expression: 
\begin{align}
|Q_{P,R}^\pi(s,a) - Q_{\bar P,\bar R}^\pi(s,a)|\leq (|1-f|+f|S|\epsilon)\frac{\gamma}{2(1-\gamma)^2}
\end{align}
Next the second term on the RHS of \cref{tring_q} can easily be bounded by $\frac{\epsilon}{1-\gamma}$ which gives:
\begin{align}
|Q_{P,R}^\pi(s,a) - Q_{\bar P,\bar R}^\pi(s,a)|&\leq (|1-f|+f|S|\epsilon)\frac{\gamma}{2(1-\gamma)^2}+   \frac{\epsilon}{1-\gamma}
\end{align}
\end{proof}

\vspace{-0.5cm}
\section{Discussion}
\subsection{Relation to other methods}
\label{discussion:vdn}
In this section we study the relationship between \textsc{Tesseract} and some of the existing methods for MARL.

\subsubsection{FQL}
FQL~\citep{chen2018factorized} uses a learnt inner product space to represent the dependence of joint Q-function on pair wise agent interactions. The following result shows containment of FQL representable action-value function by \textsc{Tesseract} : 
\begin{proposition}
The set of joint Q-functions representable by FQL is a subset of that representable by \textsc{Tesseract}.
\end{proposition}
\begin{proof}
In the most general form, any join Q-function representable by FQL has the form:
\begin{equation}
Q_{fql}(s, \mathbf{u}) = \sum_{i=1:n} q_i(s,u_i) + \sum_{i=1:n, j<i} \langle f_i(s,u_i), f_j(s,u_j) \rangle
\end{equation}
where $q_i: S\times U \to \mathbb{R}$ are individual contributions to joint Q-function and $f_i: S\times U \to \mathbb{R}^d$ are the vectors describing pairwise interactions between the agents. There are $n \choose 2$ pairs of agents to consider for (pairwise)interactions. Let $\mathscr{P} \triangleq {(i,j)}$ be the ordered set of agent pairs where $i>j$ and $i,j \in \{1..n\}$, let $\mathscr{P}_{k}$ denote the $k$th element of $\mathscr{P}$. Define membership function $m:\mathscr{P} \times \{1..n\} \to \{0,1\}$ as: 
\begin{align}
    m((i,j),x) =
    \begin{cases*}
      1     & if $x=i \vee x=j$ \\
      0     & otherwise
    \end{cases*}
\end{align}
Define the mapping $v_i: S \to \mathbb{R}^{|U|\times D}$ where $D = d {n \choose 2}+n$ and $v_{i,k}$ represents the $k$th column of $v_i$. 
\begin{align}
    v_i(s) \triangleq
    \begin{cases*}
      v_i(s)[j, (k-1)d+1:kd] = f_i(s, u_j)     & if $m(\mathscr{P}_{k}, i)=1$ \\
      v_i(s)[j, D-n+i] = q_i(s, u_j)\\
      v_i(s)[j, k] = 1     & otherwise
    \end{cases*}
\end{align}
We get that the tensors:
\begin{equation}
Q_{fql}(s) = \sum_{k=1}^D \otimes^n v_{i,k}(s)
\end{equation}
Thus any $Q_{fql}$ can be represented by \textsc{Tesseract}, note that the converse is not true ie. any arbitrary Q-function representable by \textsc{Tesseract} may not be representable by FQL as FQL cannot model higher-order ($>2$ agent) interactions. 
\end{proof}

\subsubsection{VDN}
VDN~\cite{sunehag_value-decomposition_2017} learns a decentralisable factorisation of the joint action-values by expressing it as a sum of per agent utilities $\hat Q=\oplus^n u_i ,i \in \{1..n\}$. This can be equivalently learnt in \textsc{Tesseract} by finding the best rank one projection of $\exp(\hat Q(s))$. We formalise this in the following result: 
\begin{proposition}
\label{prop:vdn}
For any MMDP, given policy $\pi$ having $Q$ function representable by VDN  ie. $\hat Q^{\pi}(s)=\oplus^n u_i(s) ,i \in \{1..n\}$, $\exists v_i(s) \forall s\in S$, the utility factorization can be recovered from rank one CP-decomposition of $\exp(\hat Q^{\pi})$
\end{proposition}
\begin{proof}
We have that :
\begin{align}
\exp(\hat Q^{\pi}(s)) &= \exp(\oplus^n u_i(s))\\
&=\otimes^n \exp(u_i(s))
\end{align}
Thus $(\exp(u_i(s)))_{i=1}^n \in \argmin_{v_i(s)}||\exp(\hat Q^{\pi}(s)) - \otimes^n v_i(s)||_F \forall s \in S$ and there always exist $v_i(s)$ that can be mapped to some $u_i(s)$ via exponentiation.
In general any Q-function that is representable by VDN can be represented by \textsc{Tesseract} under an exponential transform  (\cref{subsec:TessAlgos}).
\end{proof}

\subsection{Injecting Priors for Continuous Domains}
\label{app:cenv}
\begin{figure}[h]
\centering
\includegraphics[width=0.4\linewidth]{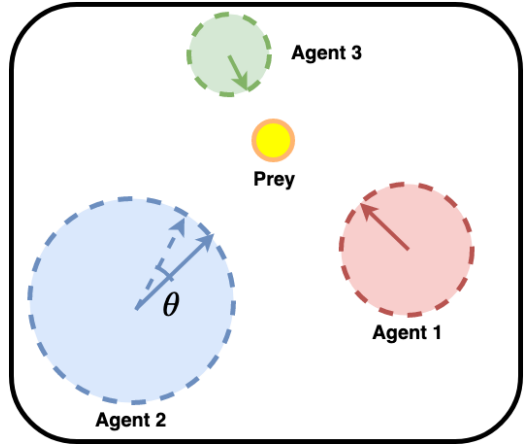}
\caption{Continuous actions task with three agents chasing a prey. Perturbing Agent 2's action direction by small amount $\theta$ leads to a small change in the joint value. \label{fig:perturb}}
\end{figure}

We now discuss the continuous action setting. Since the action set of each agent is infinite, we impose further structure while maintaining appropriate richness in the hypothesis class of the proposed action value functions. Towards this we present an example of a simple prior for \textsc{Tesseract} for continuous action domains. WLOG, let $U\triangleq \mathbb{R}^d$ for each agent $\in {1..n}$. We are now interested in the function class given by $\mathcal{Q}\triangleq \{Q: S\times U^n\to \mathbb{R}\}$ where each $Q(s)$ $\triangleq \langle T(s,\{||u^i||_2\}), \otimes^n u^i\rangle$, here $T(\cdot): S\times \mathbb{R}^n\to \mathbb{R}^{d^n}$ is a function that outputs an order $n$ tensor and is invariant to the direction of the agent actions, $\langle\cdot,\cdot\rangle$ is the dot product between two order $n$ tensors and  $||\cdot||_2$ is the Euclidean norm. Similar to the discrete case, we define 
$\mathcal{Q}_k \triangleq \{Q: Q\in \mathcal{Q} \land rank(T(\cdot)) =k, \forall s \in S\}$. The continuous case subsumes the discrete case with $T(\cdot)\triangleq Q(\cdot)$ and actions encoded as one hot vectors. We typically use rich classes like deep neural nets for $Q$ and $T$ parametrised by $\phi$.

We now briefly discuss the motivation behind the example continuous case formulation: for many real world continuous action tasks the joint payoff is much more sensitive to the magnitude of the actions than their directions, i.e., slightly perturbing the action direction of one agent while keeping others fixed changes the payoff by only a small amount (see \cref{fig:perturb}). Furthermore, $T_\phi$ can be arbitrarily rich and can be seen as representing utility per agent per action dimension, which is precisely the information required by methods for continuous action spaces that perform gradient ascent w.r.t.\ $\nabla_{u^i}Q$ to ensure policy improvement. Further magnitude constraints on actions can be easily handled by a rich enough function class for $T$. Lastly we can further abstract the interactions amongst the agents by learnable maps $f_\eta^i(u^i,s): \mathbb{R}^d \times S \to \mathbb{R}^m$, $m>>d$ and considering classes  $Q(s,\mathbf{u})\triangleq \langle T(s,\{||u^i||\}), \otimes^n f_\eta^i(u^i)\rangle$ where $T(\cdot): S\times \mathbb{R}^n\to \mathbb{R}^{m^n}$. 

\subsection{Additional theoretical discussion}
\label{app:atd}
\subsubsection{Selecting the CP-rank for approximation}
While determining the rank of a fully observed tensor is itself NP-hard \citep{hillar2013most}, we believe we can help alleviate this problem due to two key observations:
\begin{itemize}
    \item The tensors involved in \textsc{Tesseract} capture dependence of transition and reward dynamics on the action space. Thus if we can approximately identify (possibly using expert knowledge) the various aspects in which the actions available at hand affect the environment, we can get a rough idea of the rank to use for approximation.
    \item Our experiments on different domains (\cref{sec:exps}, \cref{app:additional_exp}) provide evidence that even when using a rank insufficient approximation, we can get good empirical performance and sample efficiency. (This is also evidenced by the empirical success of related algorithms like VDN which happen to be specific instances under the \textsc{Tesseract} framework.)
\end{itemize}
\section{Additional experiments and details}
\label{app:additional_exp}

\subsection{StarCraft II}
\label{app:sc2}
\begin{figure}[h]
    \centering
    \includegraphics[width=0.4\linewidth]{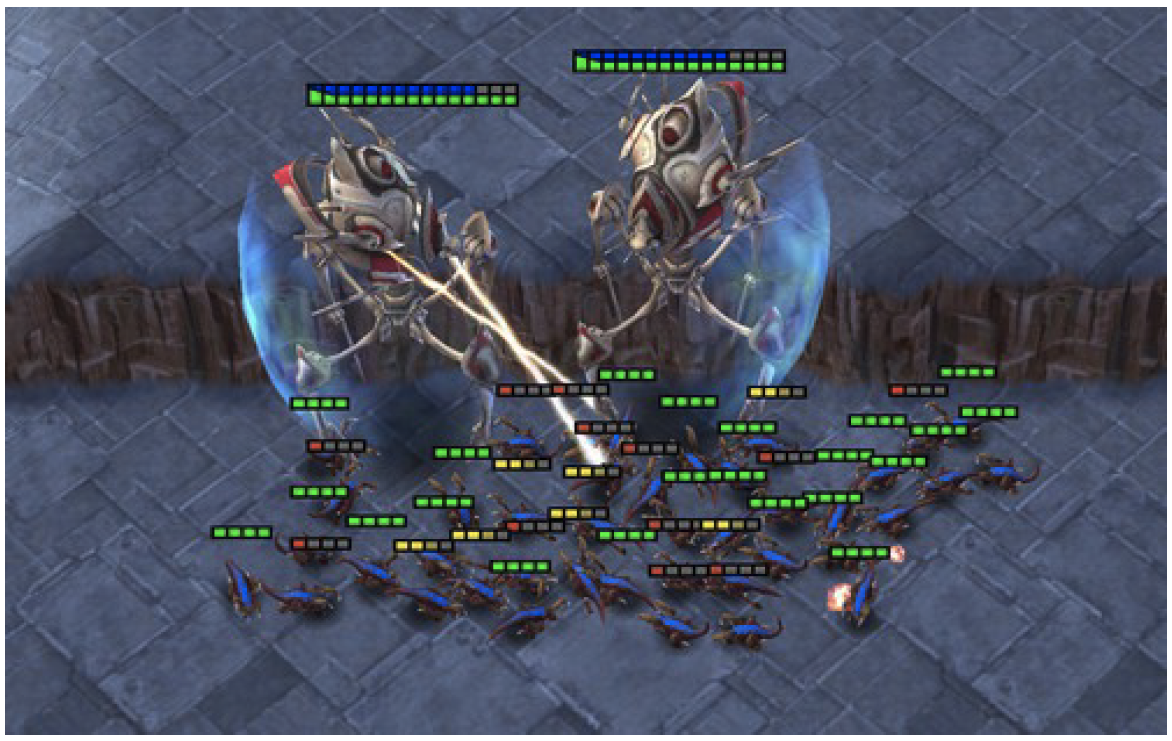}
    \caption{The 2c\_vs\_64zg scenario in SMAC. \label{fig:snap_col}}
\end{figure}

In the SMAC bechmark\citep{samvelyan2019starcraft} (https://github.com/oxwhirl/smac), agents can $\mathtt{move}$ in four cardinal directions, $\mathtt{stop}$, take $\mathtt{noop}$ (do nothing), or select an enemy to $\mathtt{attack}$ at each timestep. Therefore, if there are $n_e$ enemies in the map, the action space for each ally unit contains $n_e + 6$ discrete actions. 
\subsubsection{Additional Experiments}
\label{app:additional_sc2}
\begin{figure*}[h]
	\centering
	\subfigure[3s5z \textbf{Easy}]{
		\includegraphics[width=0.325\linewidth]{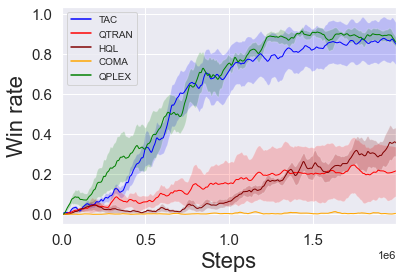}\label{fig:a_3s5z_smac}}
	\subfigure[2s\_vs\_1sc \textbf{Easy}]{
		\includegraphics[width=0.325\linewidth]{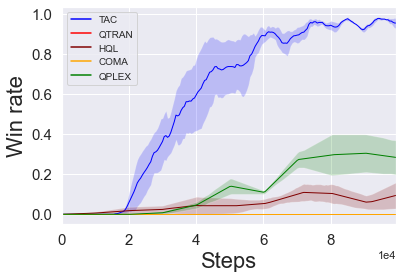}\label{fig:a_2s_vs_1sc}}
	\subfigure[2c\_vs\_64zg \textbf{Hard}]{
		\includegraphics[width=0.325\linewidth]{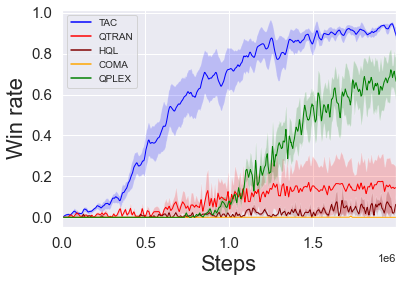}\label{fig:a_2c_vs_64z}}
	\subfigure[5m\_vs\_6m \textbf{Hard}]{
		\includegraphics[width=0.325\linewidth]{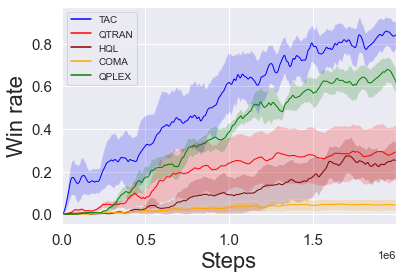}\label{fig:a_5m_vs_6m}}
	\subfigure[MMM2 \textbf{Super Hard}]{
		\includegraphics[width=0.325\linewidth]{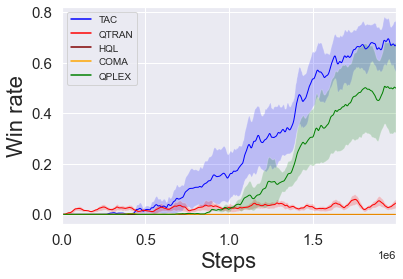}\label{fig:a_MMM2}}
	\subfigure[27m\_vs\_30m \textbf{Super Hard}]{
		\includegraphics[width=0.325\linewidth]{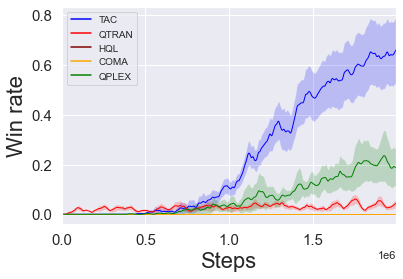}\label{fig:a_27m_vs_30m}}
	\subfigure[6h\_vs\_8z \textbf{Super Hard}]{
		\includegraphics[width=0.325\linewidth]{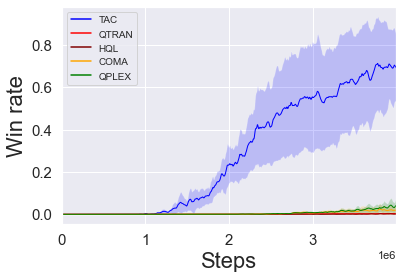}\label{fig:a_6h8z}}
	\subfigure[Corridor \textbf{Super Hard}]{
		\includegraphics[width=0.325\linewidth]{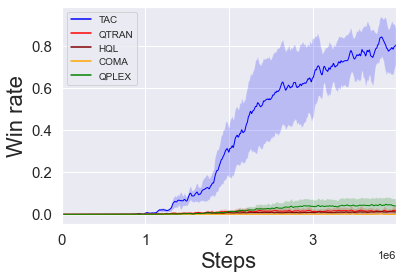}\label{fig:a_corridor}}
	\caption{Performance of different algorithms on different SMAC scenarios: \textcolor{blue}{TAC}, \textcolor{red}{QTRAN}, \textcolor[rgb]{0,0.7,0}{QPLEX}, \textcolor{orange}{COMA}, \textcolor[rgb]{0.76, 0.13, 0.28}{HQL}. \label{fig:additional_sc2}}
\end{figure*}

In addition to the baselines in main text \cref{sec:exps}, we also include 4 more baselines: \textcolor{red}{QTRAN} \citep{son2019qtran}, \textcolor[rgb]{0,0.7,0}{QPLEX} \citep{wang2020qplex}, \textcolor{orange}{COMA} \citep{foerster2018counterfactual} and \textcolor[rgb]{0.76, 0.13, 0.28}{HQL}. QTRAN tries to avoid the issues arising with representational constraints by posing the decentralised multi agent problem as optimisation with linear constraints, these constraints are relaxed using L2 penalties for tractability \citep{mahajan2019maven}. Similarly, QPLEX another recent method uses an alternative formulation using advantages for ensuring the \textit{Individual Global Max} (IGM) principle \citep{son2019qtran}. COMA is an actor-critic method that uses a centralised critic for computing a counterfactual baseline for variance reduction by marginalising across individual agent actions. Finally, HQL uses the heuristic of differential learning rates on top of IQL \citep{tan_multi-agent_1993} to address problems associated with decentralized exploration. \textbf{\cref{fig:additional_sc2} }gives the average win rates of the baselines on different SMAC scenarios across five random runs (with one standard deviation shaded). We observe that \textsc{Tesseract} outperforms the baselines by a large margin on most of the scenarios, especially on the \textbf{super-hard} ones on which the exiting methods struggle, this validates the sample efficiency and representational gains supported by our analysis. We observe that HQL is unable to learn a good policy on most scenarios, this might be due to uncertainty in the bootstrap estimates used for choosing the learning rate that confounds with difficulties arising from non-stationarity. We also observe that COMA does not yield satisfactory performance on any of the scenarios. This is possibly because it does not utilise the underlying tensor structure of the problem and suffers from a \textit{lagging critic}. While QPLEX is able to alleviate the problems arising from relaxing the IGM constraints in QTRAN, it lacks in performance on the \textbf{super-hard} scenarios of Corridor and 6h\_vs\_8z. 

\subsubsection{Experimental Setup for SMAC}
\label{app:setup_sc2}
We use a factor network for the tensorised critic which comprises of a fully connected MLP with two hidden layers of dimensions 64 and 32 respectively and outputs a $r|U|$ dimensional vector. We use an identical policy network for the actors which outputs a $|U|$ dimensional vector and a value network which outputs a scalar state-value baseline $V(s)$. The agent policies are derived using softmax over the policy network output. Similar to previous work \cite{samvelyan2019starcraft}, we use two layer network consisting of a fully-connected layer followed by GRU (of 64-dimensional hidden state) for encoding agent trajectories. We used Relu for non-linearities. All the networks are shared across the agents. We use ADAM as the optimizer with learning rate $5\times10^{-4}$. We use entropy regularisation with scaling coefficient $\beta = 0.005$. We use an approximation rank of $7$ for Tesseract ('TAC') for the SMAC experiments. A batch size of 512 is used for training which is collected across 8 parallel environments (additional setup details in \cref{app:techniques}). Grid search was performed over the hyper-parameters for tuning. 

For the baselines QPLEX, QMIX, QTRAN, VDN, COMA, IQL we use the open sourced code provided by their authors at https://github.com/wjh720/QPLEX and https://github.com/oxwhirl/pymarl respectively which has hyper-parameters tuned for SMAC domain. The choice for architecture make the experimental setup of the neural networks used across all the baselines similar. We use a similar trajectory embedding network as mentioned above for our implementations of HQL and FQL which is followed by a network comprising of a fully connected MLP with two hidden layers of dimensions 64 and 32 respectively. For HQL this network outputs $|U|$ action utilities. For FQL, it outputs  a $|U|+d$ vector: first $|U|$ dimension are used for obtaining the scalar contribution to joint Q-function and rest $d$ are used for computing interactions between agents via inner product. We use ADAM as the optimizer for these two baselines. We use differential learning rates of $\alpha = 1\times10^{-3}, \beta=2\times10^{-4}$ for HQL searched over a grid of $\{1,2,5,10\}\times10^{-3} \times\{1,2,5,10\}\times10^{-4}$. FQL uses the same learning rate $5\times10^{-4}$ with $d = 10$ which was  searched over set $\{5, 10, 15\}$. 

The baselines use $\epsilon-$greedy for exploration with $\epsilon$ annealed from $1.0 \to 0.05$ over 50K steps. For super-hard scenarios in \textbf{SMAC} we extend the anneal time to 400K steps. We use temperature annealing for \textsc{Tesseract} with temperature given by $\tau = \frac{2T}{T+t}$ where $T$ is the total step budget and $t$ is the current step. Similarly we use temperature $\tau = \frac{4T}{T+3t}$ for super-hard \textbf{SMAC} scenarios. The discount factor was set to $0.99$ for all the algorithms. 

Experiment runs take 1-5 days on a Nvidia DGX server depending on the size of the StarCraft scenario.

\subsection{Techniques for stabilising \textsc{Tesseract} critic training for Deep-MARL}
\label{app:techniques}
\begin{itemize}
    \item We used a gradient normalisation of $0.5$. The parameters exclusive to the critic were separately subject to the gradient normalisation, this was done because the ratio of gradient norms for the actor and the critic parameters can vary substantially across training. 
    \item We found that using multi-step bootstrapping substantially reduced target variance for Q-fitting and advantage estimation (we used the advantage based policy gradient $\int_{S} \rho^\pi(s)\int_{\mathbf{U}}\nabla\pi_\theta(\mathbf{u|s})\hat A^{\pi}(s,\mathbf{u}) d\mathbf{u}ds$ \citep{sutton2011reinforcement}) for \textbf{SMAC} experiments. Specifically for horizon T, we used the Q-target as:
    \begin{align}
        &Q_{target,t} = \sum_{k=1}^{T-t} \lambda^{k}g_{t,k} \\
        & g_{t,k} = R_t + \gamma R_{t+1} + ... + \gamma^{k}V(s_{t+k}) 
    \end{align}

    and similarly for value target. Likewise, the generalised advantage is estimated as:
    \begin{align}
        &\hat A_t = \sum_{k=0}^{T-t} (\gamma\lambda)^{k}\delta_{t+k}\\
        &\delta_{t} = R_{t}+ \gamma\hat Q(s_{t+1}, \mathbf{u}_{t+1})-V(s_t)
    \end{align}

    Where $\hat Q$ is the tensor network output and the estimates are normalized by the accumulated powers of $\lambda$. We used $T=64, \gamma = 0.99$ and $\lambda = 0.95$ for the experiments.
    \item The tensor network factors were squashed using a sigmoid for clipping and were scaled by $2.0$ for \textbf{SMAC} experiments. Additionally, we initialised the factors according to $\mathcal{N}(0, 0.01)$ (before applying a sigmoid transform) so that value estimates can be effectively updated without the gradient vanishing.
    \item  Similarly, we used clipping for the action-value estimates $\hat Q$ to prevent very large estimates:
    \begin{align}
    clip(\hat{Q}_t) = min\{\hat{Q}_t, R_{max}\}
    \end{align}
    we used $R_{max}=40$ for the \textbf{SMAC} experiments.
\end{itemize}

\begin{figure*}[h]
	\centering
	\subfigure[Ablation on stabilisation techniques]{
		\includegraphics[width=0.4\linewidth]{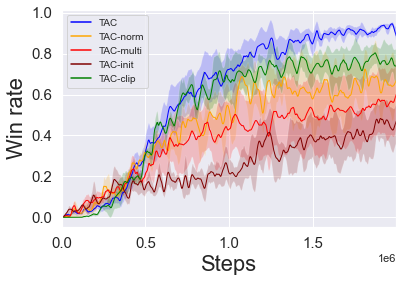}\label{fig:ab_tech}}
	\subfigure[Ablation on rank]{
		\includegraphics[width=0.4\linewidth]{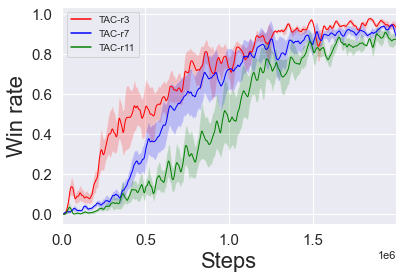}\label{fig:ab_rank}}
	\caption{Variations on \textsc{Tesseract} \label{fig:ablation_sc2}}
\end{figure*}
We provide the ablation results on the stabilisation techniques mentioned above on the 2c\_vs\_64zg scenario in \cref{fig:ab_tech}. The plot lines correspond to the ablations: \textcolor{red}{TAC-multi}: no multi-step target and advantage estimation, \textcolor[rgb]{0,0.7,0}{TAC-clip}: no value upper bounding/clipping, \textcolor{orange}{TAC-norm}: no separate gradient norm, \textcolor[rgb]{0.76, 0.13, 0.28}{TAC-init}: no initialisation and sigmoid squashing of factors. We observe that multi-step estimation of target and advantage plays a very important role in stabilising the training, this is because noisy estimates can adversely update the learn factors towards undesirable fits. Similarly, proper initialisation plays a very important role in learning the Q-tensor as otherwise a larger number of updates might be required for the network to learn the correct factorization, adversely affecting the sample efficiency. Finally we observe that max-clipping and separate gradient normalisation do impact learning, although such effects are relatively mild. 

We also provide the learning curves for \textsc{Tesseract} as the CP rank of Q-approximation is changed, \cref{fig:ab_rank} gives the learning plots as the CP-rank is varied over the set $\{3, 7, 11\}$. Here, we observe that approximation rank makes little impact on the final performance of the algorithm, however it may require more samples in learning the optimal policy. Our PAC analysis \cref{thm:debound} also supports this. 

\subsection{Tensor games:} 
\label{app:tg}
We introduce tensor games for our experimental evaluation. These games generalise the matrix games often used in $2$-player domains. Formally, a tensor game is a cooperative MARL scenario described by tuple $(n,|U|,r)$ that respectively defines the number of agents (dimensions), the number of actions per agent (size of index set) and the rank of the underlying reward tensor \cref{fig:tgg}. Each agent learns a policy for picking a value from the index set corresponding to its dimension. The joint reward is given by the entry  corresponding to the joint action picked by the agents, with the goal of finding the tensor entry corresponding to the maximum reward. We consider the CTDE setting for this game, which makes it additionally challenging. We compare \textsc{Tesseract} (\textcolor{blue}{TAC}) with \textcolor[rgb]{0,0.7,0}{VDN}, \textcolor{red}{QMIX} and independent actor-critic (\textcolor[rgb]{0.76, 0.13, 0.28}{IAC})  trained using Reinforce \cite{sutton2000policy}. Stateless games provide are ideal for isolating the effect of an exponential blowup in the action space. The natural difficulty knobs for stateless games are $|n|$ and $|U|$ which can be increased to obtain environments with  large joint action spaces. Furthermore, as the rank $r$ increases, it becomes increasingly difficult to obtain good approximations for $\hat T$.

 \begin{figure}[h]
    \centering
    \includegraphics[width=0.33\linewidth]{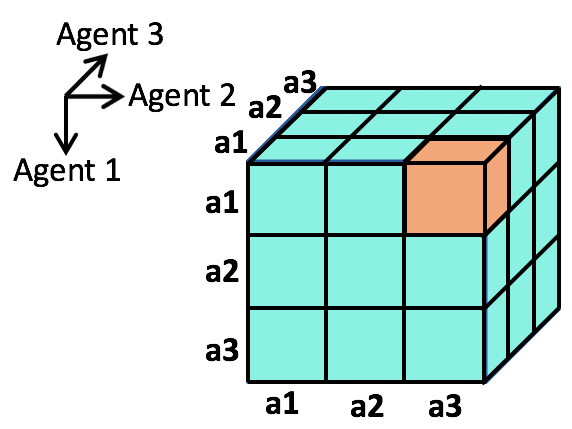}
    \caption{Tensor games example with $3$ agents ($n$) having $3$ actions each ($a$). Optimal joint-action \textbf{(a1, a3, a1)} shown in orange. \label{fig:tgg}}
\end{figure}
\begin{figure}[h]
	\centering
	\subfigure[n:5 |U|:10 r:8]{
		\includegraphics[width=0.23\columnwidth]{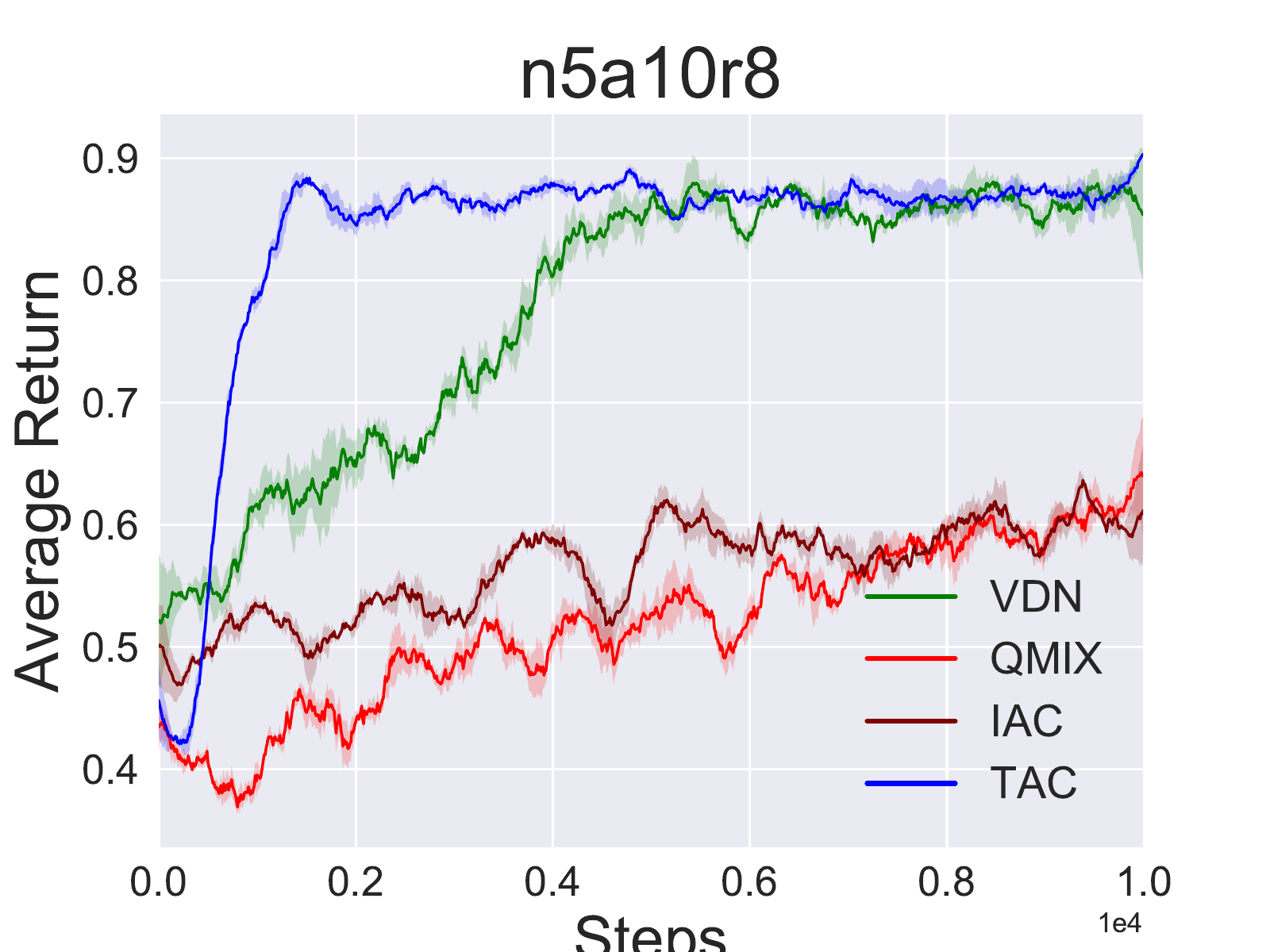}\label{fig:n5a10}}
	\subfigure[n:6 |U|:10 r:8]{
		\includegraphics[width=0.23\columnwidth]{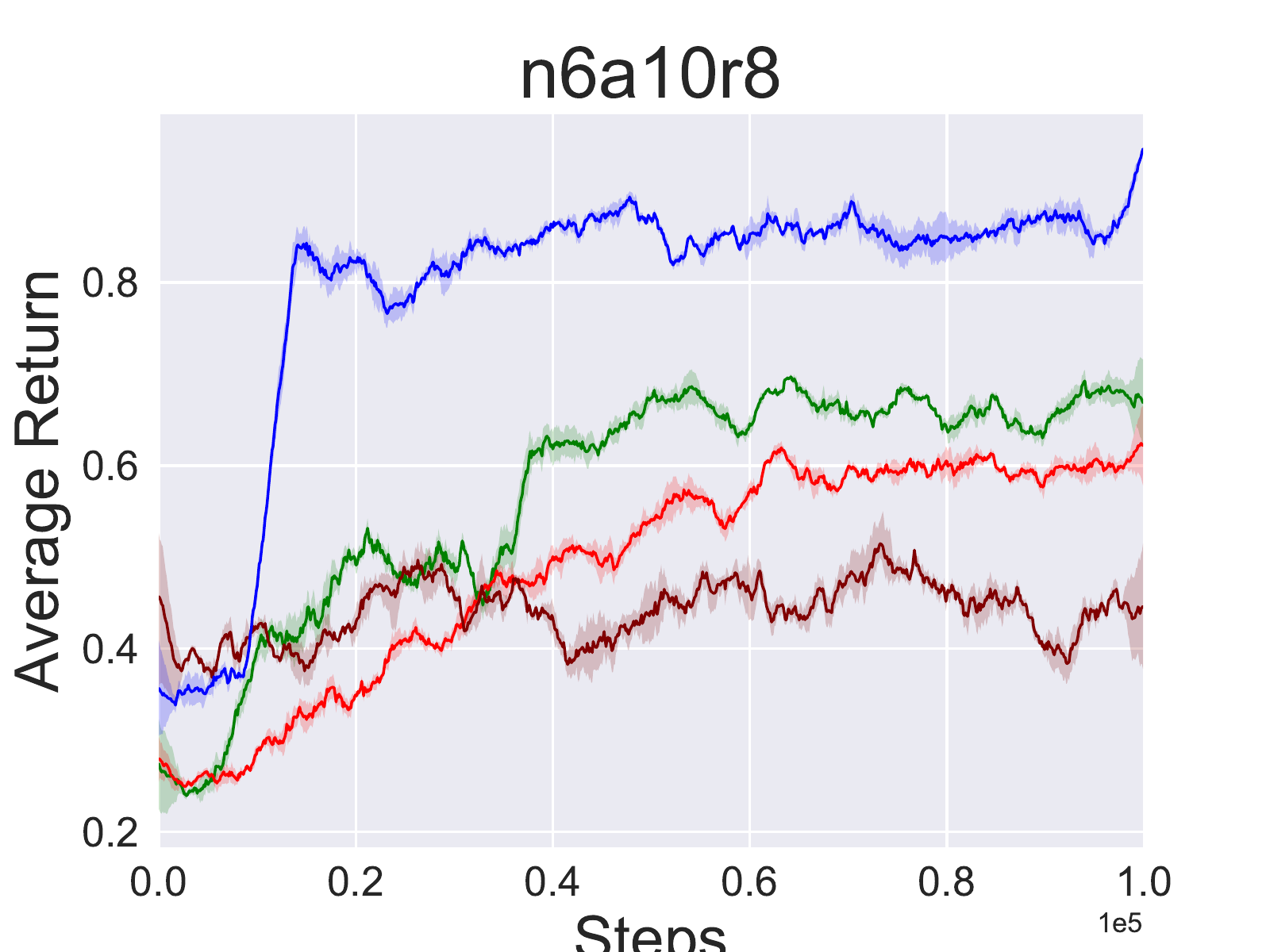}\label{fig:n6a10}}
	\subfigure[Dependence on approximation rank]{
		\includegraphics[width=0.23\columnwidth]{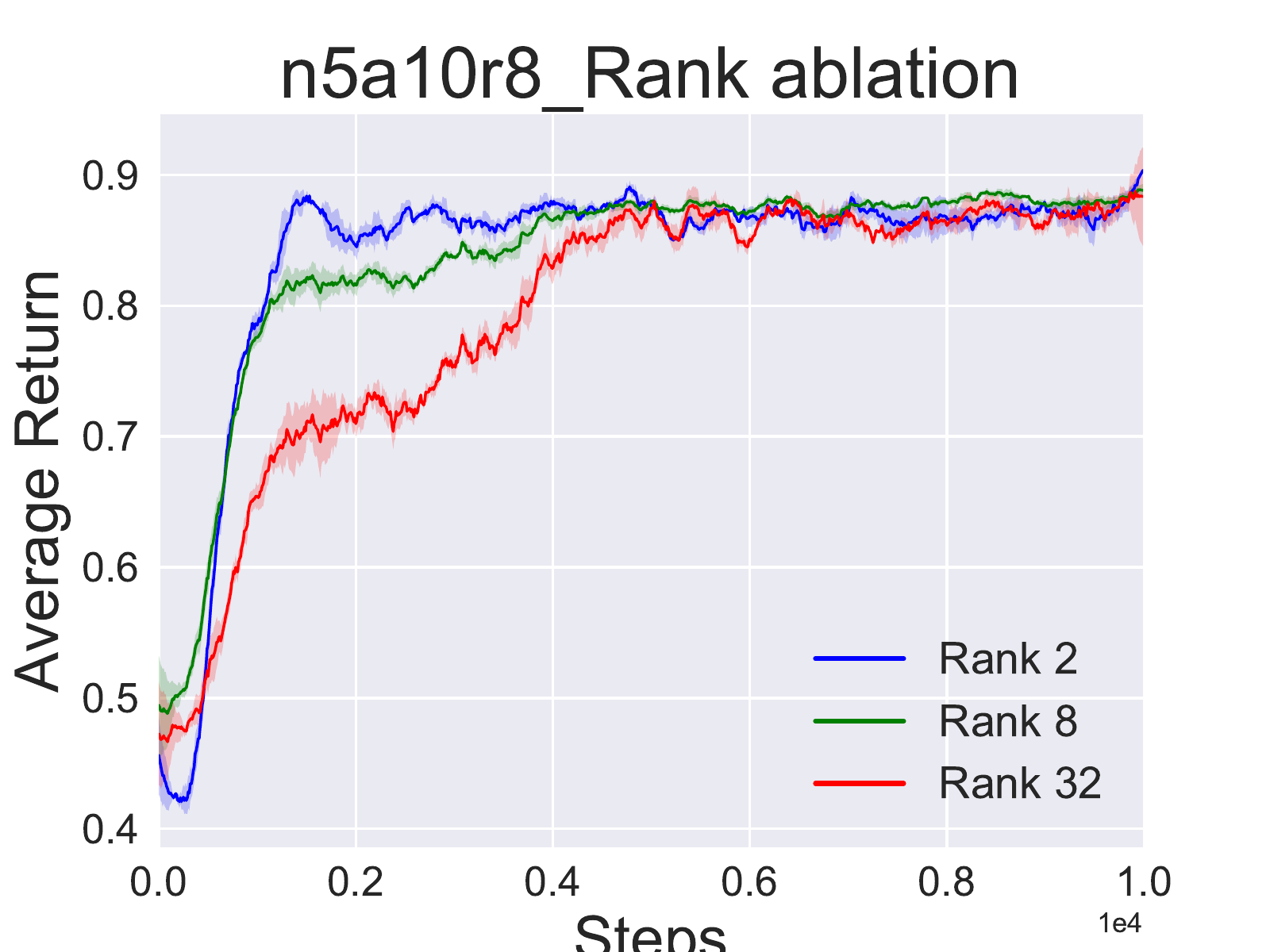}\label{fig:rab}}
	\subfigure[Effects of approximation]{
		\includegraphics[width=0.23\columnwidth]{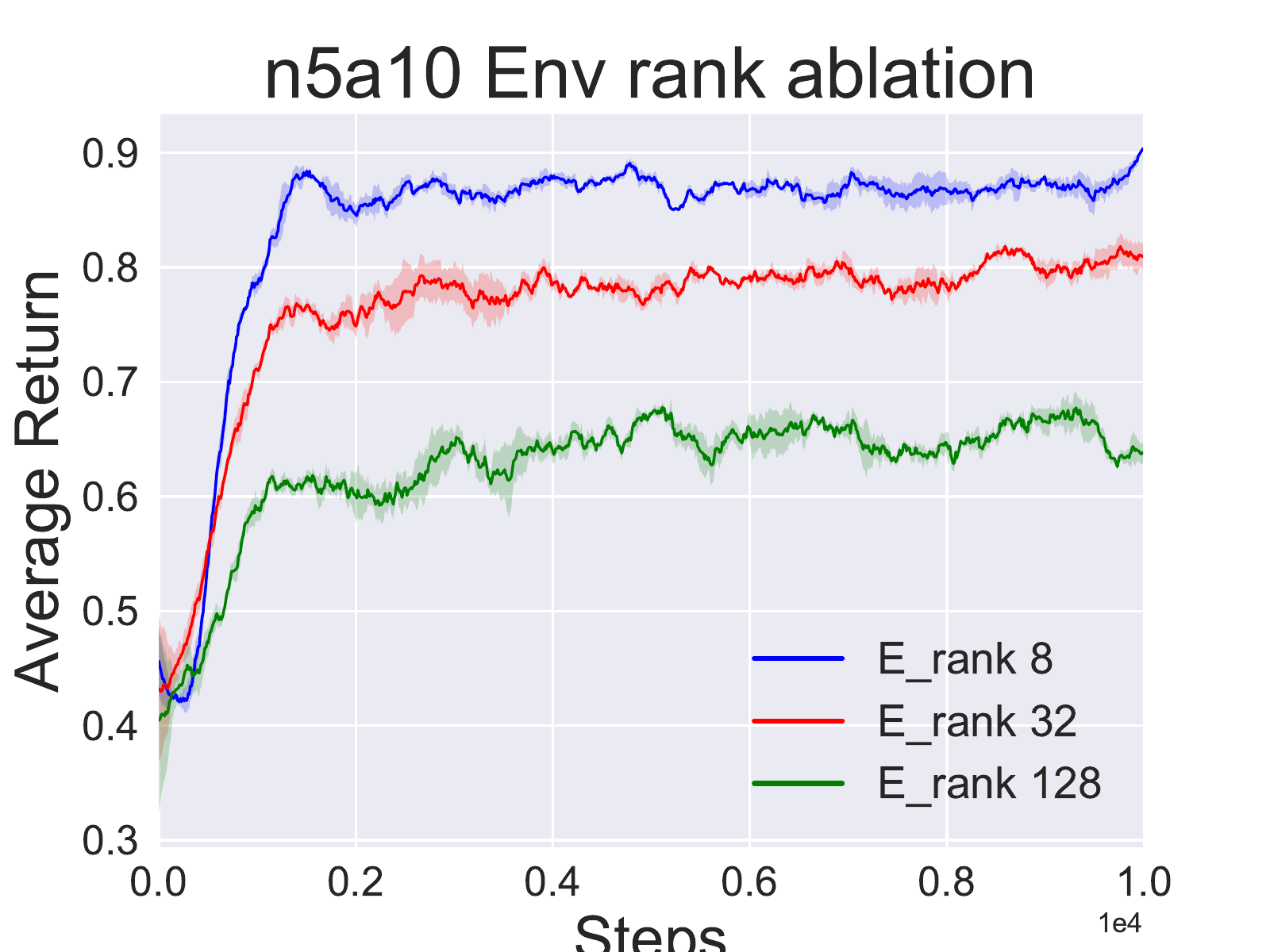}\label{fig:erab}}
	\caption{Experiments on tensor games.}
	\vspace{-0.5cm}
\end{figure}

\cref{fig:n5a10} \cref{fig:n6a10} present the learning curves for the algorithms for two game scenarios, averaged over 5 random runs with game parameters as mentioned in the figures. We observe that \textsc{Tesseract} outperforms the other algorithms in all cases. Moreover, while the other algorithms find it increasingly difficult to learn good policies, \textsc{Tesseract} is less affected by this increase in action space. As opposed to the IAC baseline, \textsc{Tesseract} quickly learns an effective low complexity critic for scaling the policy gradient. QMIX performs worse than VDN due to the additional challenge of learning the mixing network.

In \cref{fig:rab} we study the effects of increasing the approximation rank of Tesseract ($k$ in decomposition $\hat Q(s) \approx T = \sum_{r=1}^k w_r\otimes^n g_{\phi,r}(s^i) ,i \in \{1..n\},$) for a fixed environment with $5$ agents, each having $10$ actions and the environment rank being $8$. While all the three settings learn the optimal policy, it can be observed that the number of samples required to learn a good policy increases as the approximation rank is increased (notice delay in 'Rank 8', 'Rank 32' plot lines). This again is in-line with our PAC results, and makes intuitive sense as a higher rank of approximation directly implies more parameters to learn which increases the samples required to learn. 

We next study how approximation of the actual $Q$ tensors affects learning. In \cref{fig:erab} we compare the performance of using a rank-$2$ \textsc{Tesseract} approximation for environment with $5$ agents, each having $10$ actions and the environment reward tensor rank being varied from $8$ to $128$. We found that for the purpose of finding the optimal policy, \textsc{Tesseract} is fairly stable even when the environment rank is greater than the model approximation rank. However performance may drop if the rank mismatch becomes too large, as can be seen in \cref{fig:erab} for the plot lines 'E\_rank 32', 'E\_rank 128', where the actual rank required to approximate the underlying reward tensor is too high and using just $2$ factors doesn't suffice to accurately represent all the information.  

\subsubsection{Experimental setup for Tensor games}

For tensor game rewards, we sample $k$ linearly independent vectors $u_r^i$ from $|\mathcal{N}(0,1)^{|U|}|$ for each agent dimension $i \in \{1..n\}$. The reward tensor is given by $T=\sum_{r=1}^k w_r\otimes^n u_r^i ,i \in \{1..n\}$. Thus $T$ has roughly $k$ local maxima in general for $k<<|U|^n$. We normalise $\hat T$ so that the maximum entry is always $1$. 

All the agents use feed-forward neural networks with one hidden layer having $64$ units for various components. Relu is used for non-linear activation.

The training uses ADAM \citep{kingma2014adam} as the optimiser with a $L2$ regularisation of $0.001$. The learning rate is set to $0.01$. Training happens after each environment step.

The batch size is set to $32$. For an environment with $n$ agents and $a$ actions available per agent we run the training for $\frac{a^n}{10}$ steps. 

For VDN \citep{sunehag_value-decomposition_2017} and QMIX\citep{rashid2018qmix} the $\epsilon$-greedy coefficient is annealed from $0.9$ to $0.05$ at a linear rate until half of the total steps after which it is kept fixed.

For Tesseract ('TAC') and Independent Actor-Critic ('IAC') we use a learnt state baseline for reducing policy gradient variance. We also add entropy regularisation for the policy with coefficient starting at $0.1$ and halved after every $\frac{1}{10}$ of total steps.

We use an approximation rank of $2$ for Tesseract ('TAC') in all the comparisons except \cref{fig:rab} where it is varied for ablation.

\end{document}